\documentclass[12pt]{article}

\usepackage{fullpage}
\usepackage{amsfonts}
\usepackage{amssymb}
\usepackage{amstext}
\usepackage{amsmath}
\usepackage{natbib}
\usepackage{graphicx}
\usepackage{booktabs}
\usepackage{array}
\usepackage{multirow}
\usepackage{algorithm}
\usepackage{algorithmic}

\bibpunct{(}{)}{;}{a}{,}{,}
  
\usepackage{jstyle}
\usepackage{graphicx}
\usepackage{epsfig}

\def\supp{\mathop{\text{supp}\kern.2ex}}

\let\hat\widehat
\let\tilde\widetilde

\let\hat\widehat
\let\tilde\widetilde

\def\ds{\displaystyle}

\def\1{{(1)}}
\def\2{{(2)}}

\long\def\comment#1{}

\newcommand{\appsection}[1]{\let\oldthesection\thesection
  \renewcommand{\thesection}{Appendix \oldthesection}
  \section{#1}\let\thesection\oldthesection}

\begin{document}

\begin{center}
{\huge Compressive Network Analysis}\\
\vspace{.5cm}
{\Large Xiaoye Jiang$^{1}$, Yuan Yao$^{2}$, Han Liu$^{3}$, Leonidas Guibas$^{1}$}\\
\vspace{.5cm}
{\large $^{1}$Stanford University; $^{2}$Peking University;\\
$^{3}$Johns Hopkins University}\\
\vspace{.5cm}
\today
\end{center}

\begin{abstract}
Modern data acquisition routinely produces massive amounts of network data. Though many methods and models have been proposed to analyze such data, the research of network data is largely disconnected with the classical theory of statistical learning and signal processing.  In this paper, we present a new framework for modeling network data, which connects two seemingly different areas: {\it network data analysis} and {\it compressed sensing}.  From a nonparametric perspective,  we model an observed network using a large dictionary. In particular, we consider the network clique detection problem and show connections between our formulation with a new algebraic tool, namely {\it Randon basis pursuit in homogeneous spaces}. Such a connection allows us to identify rigorous recovery conditions for clique detection problems. Though this paper is mainly conceptual, we also develop practical approximation algorithms for solving empirical problems and demonstrate their usefulness on real-world datasets. 
\end{abstract}
\begin{quote}
\end{quote}


\section{Introduction}
\label{sec:intro}

In the past decade, the research of network data has increased dramatically.  Examples include scientific studies involving web data or hyper text documents connected via hyperlinks, social networks or user profiles connected via friend links,  co-authorship and citation network connected by collaboration or citation relationships, gene or protein networks connected by regulatory relationships, and much more. Such data appear frequently in modern application domains and has led to numerous high-impact applications. For instance, detecting anomaly in ad-hoc information network is vital for corporate and government security; exploring hidden community structures helps us to better conduct online advertising and marketing; inferring  large-scale gene regulatory network is crucial for new drug design and disease control. Due to the increasing importance of network data, principled analytical  and modeling tools are crucially needed.

Towards this goal, researchers from the {\it network modeling} community have proposed many models to explore and predict the network data. These models roughly fall into two categories: {\it static} and {\it dynamic} models.  For the static model, there is only one single snapshot of the network being observed. In contrast, dynamic models can be applied to analyze datasets that contain many snapshots of the network indexed by different time points.  Examples of the static network models include the Erd\"{o}s-R\'{e}nyi-Gilbert random graph model \citep{ER:59, erdos1960erg}, the $p_{1}$ \citep{Holland:81}, $p_{2}$ \citep{Duijn:2004} and more general exponential random graph (or $p^{*}$) model \citep{wasserman1996},   latent space model \citep{Hoff01latentspace}, block model \citep{Lorrain71}, stochastic blockmodel \citep{Wasserman:87}, and mixed membership stochastic blockmodel \citep{Edoardo08}.   Examples of the dynamic network models include the preferential attachment model \citep{Barabasi:99}, the small-world model \citep{watts1998cds}, duplication-attachment model \citep{Kleinberg:99, Kumar:00}, continuous time  Markov model \citep{Snijders05modelsfor}, and dynamic latent space model \citep{sarkar05b}. A comprehensive review of these models is provided in \cite{Gldenberg:09}.  

Though many methods and models have been proposed, the research of network data analysis is largely disconnected with the classical theory of statistical learning and signal processing. The main reason is that, unlike the usual scientific data for which independent measurements can be repeatedly collected, network data are in general collected in one single realization and the nodes within the network are highly relational due to the existence of many linkages. Such a disconnection prevents us from directly exploiting the state-of-the-art statistical learning methods and theory to analyze network data. To bridge this gap, we present a novel framework  to model network data. Our framework assumes that the observed network has a sparse representation with respect to some dictionary (or basis space). Once the dictionary is given, we formulate the network modeling problem into a {\it compressed sensing} problem.  Compressed sensing, also known as compressive sensing and compressive sampling, is a technique for finding sparse solutions to underdetermined linear systems. In statistical machine learning, it is related to reconstructing a signal which has a sparse representation in a large dictionary.  The field of compressed sensing has existed for decades, but recently it has exploded due to the important contributions of \cite{candes2005, candes2007, candes2008, Tsaig06compressedsensing}. By viewing the observed network adjacency matrix as the output of an underlying function evaluated on a discrete domain of network nodes,  we can formulate the network modeling problem into a compressed sensing problem.

Specifically, we consider the network clique detection problem within this novel framework.  By considering a generative model in which the observed adjacency matrix is assumed to have a sparse representation in a large dictionary where each basis corresponds to a clique, we connect our framework with a new algebraic tool, namely {\it Randon basis pursuit in homogeneous spaces}. Our problem can be regarded as an extension of the work in \cite{jaga2008} which studies sparse recovery of \emph{functions on permutation groups}, while we reconstruct \emph{functions on $k$-sets} (cliques), often called the {\em homogeneous space} associated with a permutation group in the literature~\citep{diaconis1988}. It turns out that the discrete Radon basis becomes the natural choice instead of the Fourier basis considered in~\cite{jaga2008}. This leaves us a new challenge on addressing the noiseless exact recovery and stable recovery with noise. Unfortunately, the greedy algorithm for exact recovery in \cite{jaga2008} cannot be applied to noisy settings, and in general the Radon basis does not satisfy the Restricted Isometry Property (RIP) \citep{candes2008} which is crucial for the universal recovery. In this paper, we develop new theories and algorithms which guarantee exact, sparse, and stable recovery under the choice of Radon basis. These theories have deep roots in Basis Pursuit \citep{chen1999} and its extensions with uniformly bounded noise.  Though this paper is mainly conceptual: showing the connection between network modeling and  compressed sensing,  we also provide some rigorous theoretical analysis and practical algorithms on the clique recovery problem to illustrate the usefulness of our framework.

The main content of this paper can be summarized as follows. Section \ref{sec:mainidea} presents the general framework on compressive network analysis. In Section \ref{sec:cliquedetection}, \ref{sec:radonbp} and \ref{sec:maththeory}, we consider the clique detection problem under the compressive network analysis framework. A polynomial time approximation algorithm is provided in Section \ref{sec:algorithm} for the clique detection problem. We also demonstrate successful application examples in Section \ref{sec:application}. Section \ref{sec:conclusion} concludes the paper. 

\section{Main Idea}\label{sec:mainidea}

In this section we present the general framework of compressive network analysis with a nonparametric view.  We start with an introduction of notations:   let $u  = (u_{1}, \ldots, u_{d})^{T} \in \mathbb{R}^{d}$ be a vector and $I(\cdot)$ be the indicator function. We denote 
\begin{eqnarray}
\| u\|_{0} \equiv \sum_{j=1}^{d}I(u_{j}\neq 0),~~\| u\|_{2} \equiv \sqrt{\sum_{j=1}^{d}u^{2}_{j}}, ~~\| u\|_{\infty} \equiv \max_{j}|u_{j}|.
\end{eqnarray}
We also denote by $\langle \cdot, \cdot \rangle$ the Euclidean inner product and $\mathrm{sign}(u) = (\mathrm{sign}(u_{1}),\ldots, \mathrm{sign}(u_{d}))^{T}$, where
\begin{eqnarray}
 \mathrm{sign}(u_{j}) =
\left\{
\begin{array}{ccc}
  +1   &  u_{j} > 0 \\
  0   &   u_{j} = 0~ \\
  -1   &  u_{j}<0
\end{array}
\right.
\end{eqnarray}

We represent a network as a graph $G=(V,E)$, where $V=\{1,\ldots, n\}$ is the set of  nodes and $E\subset V\times V$ is the set of edges.  Let $B \in \mathbb{R}^{n\times n}$ be the adjacency matrix of the observed network with $B_{ij}$ represents a quantity associated with nodes $i$ and  $j$. With no loss of generality, we assume  that $B$  is symmetric: $B = B^{T}$ and $\mathrm{diag}(B)=0$.  With these assumptions, to model $B$ we only need to model its upper-triangle. For notational simplicity, we squeeze $B$ into a vector $b \in \mathbb{R}^{M}$ where $M = n(n-1)/2$ is the number of upper-triangle elements in $B$.   Let $f(V) \in \mathbb{R}^{M}$ be an unknown vector-valued function defined on $V$. We assume a generative model of the observed adjacency matrix $B$ (or equivalently, $b$):
\begin{eqnarray}
b = f(V) + z, \label{eq::general}
\end{eqnarray}
where $z \in \mathbb{R}^{M}$ is a noise vector. We can view $f(V)$ as evaluating a possibly infinite-dimensional function $f$ on a discrete set $V$,  thus the model \eqref{eq::general} is intrinsically  nonparametric and can model any static networks.

Without further regularity conditions or constraints, there is no hope for us to reliably estimate $f$. In our framework, we  assume that $f$ has a sparse representation with respect to an $M$ by $N$ dictionary $A=[\phi_{1}(V), \ldots, \phi_{N}(V)]$ where each $\phi_j(V)\in \mathbb{R}^{M}$ is a basis function, i.e., there exists a subset $S \subset \{1, \ldots, N\}$ with cardinality $|S| \ll N$, such that 
\begin{eqnarray}
f(V) = \sum_{q\in S} x_{q} \phi_{q}(V).
\end{eqnarray}
In the sequel, we denote by $A_{pq}$ the element on the $p$-th row and $q$-th column of $A$. Here $p$ indexes a pair of different nodes and $q$ indexes a basis $\phi_{q}(V)$.   To estimate $f$, we only need to reconstruct $x=(x_{1}, \ldots, x_{N})^{T}$. Given the dictionary $A$, we can estimate $f$ by solving  the following program:
\begin{eqnarray}\label{eq:P0}
{\rm (P_0)}&&~\min\|x\|_{0}~\text{s.t.}~\| b - Ax\|_{z} \leq \delta
\end{eqnarray}
where $\|\cdot\|_{z}$ is a vector norm constructed using the knowledge of $z$. The problem in \eqref{eq:P0} is non-convex.  In the sparse learning literature, a convex relaxation of \eqref{eq:P0} can be written as
\begin{eqnarray}\label{eq:P1}
{\rm (P_1)}&&~\min\|x\|_{1}~\text{s.t.}~\| b-  Ax\|_{z} \leq \delta.
\end{eqnarray}
One thing to note is that the dictionary $A$ can be either constructed based on the domain knowledge, or it can be learned from empirical data. For simplicity, we always assume $A$ is pre-given in this paper.  In the following sections, we use the clique detection problem as a case study to illustrate the usefulness of this framework.

\section{Clique Detection}\label{sec:cliquedetection}

In network data analysis, The problem of identifying communities or cliques\footnote{A clique means a complete subgraph of the network.} based on partial information arises frequently in many applications,  including identity management~\citep{guibas2008}, statistical ranking~\citep{diaconis1988, jaga2008}, and social networks~\citep{leskovec2010}. In these applications we are typically given a network with its nodes representing players, items, or characters, and edge weights summarizing the observed pairwise interactions. 
{\it The basic problem is to determine communities or cliques within the network by observing the frequencies of low order interactions,} since in reality such low order interactions are often governed by a considerably smaller number of high order
communities or cliques. Therefore  the clique detection problem can be formulated as \emph{compressed sensing} of cliques in large networks. To solve this problem, one has to answer two questions: {\it (i) what is the suitable representation
basis, and (ii) what is the reconstruction method?} 
 Before rigorously formulating the problem, we provide three motivating examples as a glimpse of typical situations which can be addressed within the framework in this paper.

\textbf{Example 1 (Tracking Team Identities)}
We consider the scenario of multiple targets moving in an environment monitored by sensors.
We assume every moving target has an identity and
they each belong to some teams or groups.
However, we can only obtain partial interaction information due
to the measurement structure. For example, watching a grey-scale video of a basketball
game (when it may be hard to tell apart the two teams), sensors may observe
ball passes or collaboratively offensive/defensive interactions
between teammates. The observations are partial due to the fact that
players mostly exhibit to sensors low order interactions in basketball games.
It is difficult to observe a single event which involves all team members.
Our objective is to infer membership information (which team the players belong to)
from such partially observed interactions.

\textbf{Example 2 (Inferring High Order Partial Rankings)}
The problem of clique identification also arises in ranking problems.
Consider a collection of items which are to be ranked by a set of users.
Each user can propose the set of his or her $j$ most favorite items (say top 3 items) but without specifying a relative preference within this set.
We then wish to infer what are the top $k>j$ most favorite items (say top 5 items).
This problem requires us to infer high order partial rankings from low order observations.

\textbf{Example 3 (Detecting Communities in Social Networks)}
Detecting communities in social networks is of extraordinary importance.
It can be used to understand the organization or collaboration structure of a social network. However, we do not have direct mechanisms to sense social communities. Instead, we have partial, low order interaction information. For example, we observe pairwise or triple-wise co-appearance among people who hang out for some leisure activities together. We hope to detect those social communities in the  network from such partially observation data.

In these examples we are typically given a network with some nodes representing players, items, or characters, and edge weights summarizing the observed pairwise interactions. Triple-wise and other low order information can be further exploited if we consider complete sub-graphs or cliques in the networks.
{\it The basic problem is to determine common interest groups or cliques within the network by observing the frequency of low order interactions}. Since in reality such low order interactions are often governed by a considerably smaller number of high order
communities. In this sense we shall formulate our problem as \emph{compressed sensing of cliques} in networks.

The problem we are going to address has a close relationship with community detection in social networks. Community structures are ubiquitous in social networks. However, there is no consistent definition of a ``community". In the majority of research studies, community detections based on partitions of nodes in a network. Among these works, the most famous one is based on the modularity~\citep{Newman06} of a partition of the nodes in a group. A shortcoming in partition-based methods is that they do not allow overlapping communities, which occur frequently in practice.  Recently there has been growing interest in studying overlapping community structures~\citep{Lanci09}. The relevance of cliques to overlapping communities was probably first addressed in the clique percolation method~\citep{Palla05}.  In that work, communities were modeled as maximal connected components of cliques in a graph where two $k$-cliques are said to be connected if they share $k-1$ nodes. In this paper, we pursue a compressive representation of signals or functions on networks based on clique information which in turns sheds light on multiple aspects of community
structure.

In this paper,  we use the same definition as in \cite{Palla05} but are more interested in identifying cliques.  We pursue an alternative approach on  exploring networks based on clique information which potentially sheds light on multiple aspects of community
structures. Roughly speaking, we assume that there is a frequency function defined on complete low order subsets. For example, in some social networks edge weights are bivariate functions defined on pairs of nodes reflecting strength of pairwise interactions. We also assume that there is another latent frequency function defined on complete high order subsets which we hope to infer. Intuitively, the interaction frequency of a particular low order subset should be the sum of frequencies of high order subsets which it belongs to. Hence we consider a {\em generative mechanism} in which there exists a linear mapping from frequencies on high order subsets (usually sparsely distributed) to low order subsets. One typically can collect data on low order subsets while the task is to find those few dominant high order subsets. This problem naturally fits into the general compressive network analysis framework we introduced in the previous section. Below we demonstrate that the Radon basis will be an appropriate representation for our purpose which allows the sparse recovery by a simple linear programming reconstruction approach.

\section{Radon Basis Pursuit}\label{sec:radonbp}

\subsection{Mathematical Formulation}

Under the general framework in \eqref{eq::general}, we formulate the clique detection problem into a compressed sensing problem named {\it Radon Basis Pursuit}. For this, we construct a dictionary $A$ so that each column of $A$  corresponds to one clique.  The intuition of such a construction is that we assume there are several hidden cliques within the network,  which are perhaps of different sizes and may have overlaps.  Every clique has certain weights. The observed adjacency matrix $B$ (or equivalently, its vectorized version $b$) is a linear combination of many clique basis contaminated by a noise vector $\epsilon$. 

For simplicity, we first restrict ourselves to the case that all the cliques are of the same size $k < n$. The case with mixed sizes will be discussed later.  Let $C_{1}, C_{2}, \ldots, C_{N} $ be all the cliques of size $k$ and each $C_{j} \subset V$. We have $N =  {n \choose k}$.  For each $q \in \{1,\ldots, N\}$, we construct the dictionary $A$ as the following
\begin{eqnarray}
A_{pq} = \left\{
\begin{array}{cl}
 1  & \text{if the $p$-th pair of nodes both lie in $C_{q}$}     \\
 0  & \text{otherwise}.  
\end{array}
 \right. \nonumber
\end{eqnarray}

The matrix $A$ constructed above is related to discrete Radon transforms. In fact, up to a
constant and column scaling, the transpose matrix $A^{*}$ 
is called the discrete Radon transform for two suitably defined homogeneous spaces~\citep{diaconis1988}. 
Our usage here is to exploit the transpose matrix of the Radon transform to construct an over-complete dictionary,
so that the observed output $b$ has a sparse representation with respect to it. More technical discussions of the Radon transforms is beyond the scope of this paper. 

The above formulation can be generalized to the case where $b$ is a vector of length ${n\choose j}$ ($j\ge 2$) with the $p$'th entry in $b$ characterizing a quantity associated with a $j$-set (a set with cardinality $j$). The dictionary $A$ will then be  a binary matrix $R^{j,k}$ with entries indicating whether a $j$-set is a subset of a $k$-clique (a clique with $k$ nodes), i.e.,
\begin{eqnarray}
R^{j,k}_{pq} = \left\{
\begin{array}{cl}
 1  & \text{if the $p$-th $j$-set of nodes all lie in the $k$-clique $C_{q}$}     \\
 0  & \text{otherwise}.  
\end{array}
 \right. \nonumber
\end{eqnarray}
Therefore, the case where $b$ is the vector of length ${n \choose 2}$ corresponds to a special case where $A=R^{2,k}$. Our algorithms and theory hold for general $R^{j,k}$ with $j<k$.

Now we provide two concrete reconstruction programs for the clique identification problems:
\begin{eqnarray*}\label{eq:P1}
{\rm (\mathcal{P}_1)}&&~\min\|x\|_{1}~\ \text{s.t.}~b =   Ax \\
{\rm (\mathcal{P}_{1,\delta})}&&~\min\|x\|_{1}\ ~\text{s.t.}~\| Ax-b\|_{\infty} \leq \delta. 
\end{eqnarray*}

 $\mathcal{P}_1$ is known as Basis Pursuit \citep{chen1999} where we consider an ideal case that the noise level is zero.  For
robust reconstruction against noise, we consider the relaxed program $\mathcal{P}_{1,\delta}$. 
The program in $\mathcal{P}_{1,\delta}$ differs from  the Dantzig selector \citep{candes2007} which uses the constraint in the form $\|A^* (Ax - b )\|_\infty \leq \delta$.   The reason for our choice of $\mathcal{P}_{1,\delta}$ lies in the fact that a more natural noise model for network data is bounded noise rather than Gaussian noise. Moreover, our linear programming formulation of $\mathcal{P}_{1,\delta}$ enables practical computation for large scale problems.

\subsection{Intuition}

Let $G=(V,E)$ be the network we are trying to model. The set of vertices $V$ represents individual identities such as
people in the social network. Each edge in $E$ is associated with some weights which represent interaction frequency information.

We assume that there are several common interest groups or communities within the network, represented by cliques (or complete sub-graphs) within graph $G$, which are perhaps of different sizes and may have overlaps. Every community has certain interaction frequency which can be
viewed as a function on cliques. However, we only receive partial measurements consisting of
low order interaction frequency on subsets in a clique. For example, in the simplest case we may only
observe pairwise interactions represented by edge weights. Our problem is to reconstruct the function on cliques from such partially observed data. A graphical illustration of this idea is provided in Figure \ref{fig:example1}, in which we see an observed network can be written as a linear combination of several overlapped cliques.

\begin{figure}[htp]
\begin{center}
\includegraphics[width=.85\textwidth]{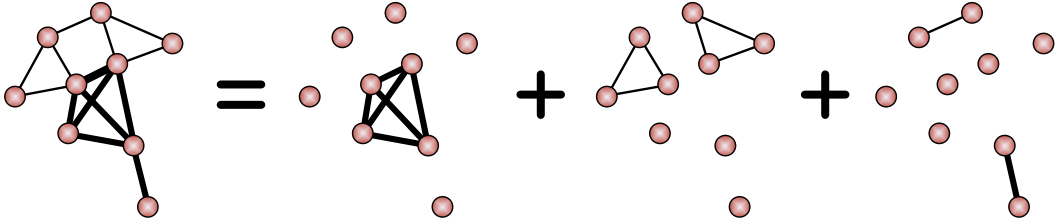}
\caption{ \label{fig:example1} An illustrative example of the main idea.}
\end{center}
\end{figure}

One application scenario is to identify two basketball teams from pairwise interactions among players. 
Suppose we  have $x_0$ which is a signal on all $5$-sets of a $10$-player set. We assume it is sparsely
concentrated on two $5$-sets which correspond to the two teams with nonzero weights. 
Assume we have observations $b$ of pairwise interactions $b=Ax_0+z$, where $z$ is uniform
random noise defined on $[-\epsilon, \epsilon]$. We solve $\mathcal{P}_{1,\delta}$, with $\delta=\epsilon$, which is a linear program over $x\in \mathbb{R}^{10\choose 5} = \mathbb{R}^{252}$ with parameters $A\in \mathbb{R}^{{10\choose 2}\times{10 \choose 5}} = \mathbb{R}^{45\times 252}$ and $b\in \mathbb{R}^{45}$.

\subsection{Connection with Radon Basis}

Let $V_j$ denote the set of all $j$-sets of $V=\{1,\cdots, n\}$ and $M^j$ be the set of real-valued functions on $V_j$.
The observed interaction frequencies $b$ on all $j$-sets, can be viewed
as a function in $M^{j}$. We build a matrix $\tilde{R}^{j,k}:M^k \to M^j$ ($j<k$) as a mapping from functions on all $k$-sets of $V$ to functions
on all $j$-sets of $V$. In this setup, each row represents a $j$-set and each column represents a $k$-set. The
entries of $\tilde{R}^{j,k}$ are either $0$ or $1$ indicating whether the $j$-set is a subset of
the $k$-set. Note that every column of $\tilde{R}^{j,k}$ has {\scriptsize ${k\choose j}$} ones.
Lacking {\it a priori} information, we assume that every $j$-set of a particular $k$-set has equal interaction probability,
whence choose the same constant $1$ for each column.
We further normalize $\tilde{R}^{j,k}$ to $R^{j,k}$ so that the $\ell_2$ norm of each column of $R^{j,k}$ is $1$. To summarize, we have
\begin{equation*} \label{eq:radonmat}\scriptsize
R^{j,k}_{(\sigma,\tau)}=\left\{\begin{array}{cc} \ds  \frac{1}{\sqrt{ {k \choose j}}}, & \mbox{if } \ds \sigma\subset \tau; \\ 0, & \mbox{otherwise},\end{array} \right.
\end{equation*}
where $\sigma$ is a $j$-set and $\tau$ is a $k$-set. As we will see, this construction leads to a canonical basis associated with the discrete Radon transform.
The size of matrix $R^{j,k}$ clearly depends on the total number of items $n=|V|$. We omit $n$ as its meaning will be clear from the context.

The matrix $R^{j,k}$ constructed above is related to discrete Radon transforms on homogeneous space $M^k$. In fact, up to a
constant, the adjoint operator $(R^{j,k})^*$ 
is called the discrete Radon transform from homogeneous space
$M^j$ to $M^k$ in~\cite{diaconis1988}.  Here all the $k$-sets form a homogeneous space.  The collection of all row vectors of $R^{j,k}$ is called as the $j$-th \emph{Radon basis} for $M^k$.
Our usage here is to exploit the transpose matrix of the Radon transform to construct an over-complete dictionary for $M^j$,
so that the observation $b$ can be represented by a possibly sparse function $x\in M^k$ ($k\geq j$).


The Radon basis was proposed as an efficient way to study partially ranked data in \cite{diaconis1988}, where
it was shown that by looking at low order Radon coefficients of a function on $M^k$, we usually get useful and interpretable
information. The approach here adds a reversal of this perspective, i.e. the reconstruction of sparse high order functions from low order Radon coefficients. We will discuss this in the sequel with a connection to the compressive sensing \citep{chen1999, candes2005}.

\section{Mathematical Theory}\label{sec:maththeory}

One advantage of our new framework on compressive network analysis  is that it enables rigorous theoretical analysis of the corresponding convex programs. 

\subsection{Failure of Universal Recovery}
Recently it was shown by \cite{candes2005} and \cite{candes2008} that $\mathcal{P}_1$ has a unique sparse solution $x_0$, if the matrix $A$ satisfies the \emph{Restricted Isometry Property} (RIP), i.e.
for every subset of columns $T \subset \{1,\ldots, N\}$ with $|T|\leq s$, there exists a certain universal constant $\delta_s\in [0,\sqrt{2}-1)$ such that
\[ (1-\delta_s)\|x\|_{2}^2 \le \|A_Tx\|_{2}^2 \le (1+\delta_s)\|x\|_{2}^2, \ \ \ \  \forall x\in \mathbb{R}^{|T|}, \]
where $A_{T}$ is the sub-matrix of $A$ with columns indexed by $T$. Then exact recovery holds for all $s$-sparse signals $x_0$ (i.e. $x_{0}$ has at most $s$ non-zero components),  whence called the \emph{universal recovery}.

Unfortunately, in our construction of the basis matrix $A$, RIP is not satisfied unless for very small $s$. The following theorem illustrates the failure of universal recovery in our case. 

\begin{theorem}
Let $n>k+j+1$ and $A=R^{j,k}$  with $j<k$. Unless $s< {k+j+1\choose k}$, there does not exist a $\delta_{s}<1$ such that the inequalities 
$$(1-\delta_s)\|x\|_{2}^2 \le \|A_Tx\|_{2}^2 \le (1+\delta_s)\|x\|_{2}^2,~~ \forall x\in \mathbb{R}^{|T|}$$
hold universally for every  $T\subset \{1,\ldots, N\}$ with $|T|\leq s$, where $N ={n\choose k}$.
\end{theorem}

Note that ${k+j+1\choose k}$ does not depend on the network size $n$, which will be problematic.  We can only recover a constant number of cliques no matter how large the network is.  The main problem for such a negative result is that the RIP tries to guarantee exact recovery for {\it arbitrary} signals with a sparse representation in $A$.  For many applications, such a condition is too strong to be realistic. Instead of studying such ``universal'' conditions,  in this paper we seek conditions that secure exact recovery of a collection of sparse signals $x_0$,  whose sparsity pattern satisfies certain conditions more appropriate to our setting. Such conditions could be more natural in reality, which will be shown in the sequel as simply requiring bounded overlaps between cliques. 
 
\begin{remark}
Recall that the matrix $A$ has altogether $N = {n\choose k}$ columns. Each column in fact corresponds to a $k$-clique. Therefore, we could also use a $k$-clique to index a column of $A$.  In this sense, let $T = \{i_{1},\ldots, i_{k}\} \subset \{1, \ldots, N\}$ be a  subset of size $k$. An equivalent notation is to represent $T$ as a class of sets: $T=\{\tau_{1}, \ldots, \tau_{k}\}$ where each $\tau_{i}\subset \{1,\ldots, n\}$ and $|\tau| = k$.
\end{remark}

 \begin{proof} 
 We can extract a set of columns $T=\{\tau : \tau\subset \{1,2,\cdots, k+j+1\}~\text{and}~|\tau|=k\}$ ($\tau$ is
interpreted as a $k$-set) and form a submatrix $A_T$. 
Recall that $A$ has altogether ${n\choose j}$ number of rows. Combined with the condition that $n>k+j+1$ and the fact that the number of nonzero rows of $A_{T}$ should be exactly ${k+j+1 \choose j}$. We know that there must exist rows in $A_{T}$ which only contains zeroes.

By discarding zero rows, it is easy
to show that the rank of $A_T$ is at most ${k+j+1\choose j}$, which is less than the number of columns. To see that the rank of $A_T$ is at most ${k+j+1\choose j}$, we need to exploit the fact that $j < k$, therefore
\begin{eqnarray}
{k+j+1\choose j} <{k+j+1\choose k},
\end{eqnarray}
from which we see that the number of nonzero rows of $A_{T}$ is smaller than the number of columns. 

Thus, the columns in $A_T$ must be linearly dependent. In other words,
there exist a nonzero vector $h \in \mathbb{R}^{N}$ where $\mathrm{supp}(h)\subset T$ such that $Ah=0$.
When $s\ge{k+j+1\choose k}$, Since $|\mathrm{supp}(h)|\leq |T| < s$, we can not expect universal sparse recovery for all $s$-sparse signals .
 \end{proof}

\subsection{Exact Recovery Conditions}

Here we present  our exact recovery conditions for $x_0$ from the observed data $b$ by solving the linear program $\mathcal{P}_1$. Suppose $A$ is an $M$-by-$N$ matrix and $x_0$ is a sparse signal.
Let $T=\mathrm{supp}(x_0)$, $T^c$ be the complement of $T$, and $A_T$ (or $A_{T^c}$) be the submatrix of $A$ where we only extract column set $T$ (or $T^c$, respectively).  
The following proposition from \cite{candes2005} characterizes the conditions that $\mathcal{P}_{1}$ has  a unique condition. To make this paper self-contained, we also include the proof in this section.
\begin{proposition}\label{pro::P1} 
{\rm \citep{candes2005}} Let $x_{0} = (x_{01}, \ldots, x_{0N})^{T}$, we assume that $A_T^*A_T$ is invertible and there exists a vector $w\in \mathbb{R}^M$ such that
\begin{enumerate}
\item $\left<A_j, w \right>= \mathrm{sign}(x_{0j}), \forall j \in T$; 
\item $|\left<A_j, w \right>| < 1, \forall j \in T^c$.
\end{enumerate}
Then $x_0$ is the unique solution for $\mathcal{P}_1$.
\end{proposition}

\begin{proof} The necessity of the two conditions come from the KKT conditions of $\mathcal{P}_1$. If we consider an equivalent form of $\mathcal{P}_1$

\begin{eqnarray*}
\min & & 1^T \xi \\
\mbox{subject to}&  & A x - b = 0 \\
& & -\xi \leq x \leq \xi \\
& & \xi \geq 0
\end{eqnarray*}
whose Lagrangian is
\[ L(x,\xi; \gamma,\lambda,\mu) = 1^T\xi + \gamma^T(Ax - b)  - \lambda_+^T(\xi - x) - \lambda_- ^T(\xi + x) - \mu^T \xi. \]
Here $\gamma \in \mathbb{R}^{M}$, $\lambda_{+} = (\lambda_{+}(1), \ldots, \lambda_{+}(N))^{T}\in\mathbb{R}^{N}_{+}$, $\lambda_{-} = (\lambda_{-}(1), \ldots, \lambda_{-}(N))^{T}\in\mathbb{R}^{N}_{+}$, $\mu \in \mathbb{R}^{N}_{+}$ are the Lagrange multipliers.

Then the KKT condition gives
\begin{enumerate}
 \item $A^* \gamma + (\lambda_+ - \lambda_- )= 0$, 
 \item $1 - (\lambda_+ + \lambda_-) - \mu = 0$,
\end{enumerate}
with $\lambda, \mu \geq 0$ and $\lambda_+(j)\lambda_-(j)=0$ for all $j$. 

Clearly $T = \mathrm{supp}(x_0) =\{ j: \xi_j > 0\}$. Let $w = \gamma$, 
by the Strictly Complementary Theorem for linear programming in \cite{yebook1997},  there exist $\mu$ and $\xi$ such that $1 >\mu_j>0$ for all $j \in T^c$ with $\xi_j = 0$, and $\mu_j=0$ for all $j\in T$ with $\xi_j >0$. Thus, the first equation leads to 
\[  \langle w, A_j\rangle =  -(\lambda_+(j) - \lambda_-(j))= -\mathrm{sign} (x_{0j}), \ \ \ \ \ j \in T; \]
the second equation leads to 
\[ |\langle w, A_j\rangle | = |\lambda_+(j) - \lambda_-(j)| = 1 - \mu_j < 1. \]
Therefore, the two conditions are necessary for $x_0$ to be the unique solution of $\mathcal{P}_{1}$.

To prove that these two conditions are sufficient to guarantee $x_0$ is the unique minimizer to $\mathcal{P}_1$, we need to show any minimizer $y_0$ to the problem $\mathcal{P}_1$ must be equal to $x_0$. Since $x_0$ obeys the constraint $Ax_0=b$, we must have $$\|y_0\|_1\leq \|x_0\|_1.$$

Now take a $w$ obeying the two conditions, we then compute
\begin{eqnarray*}
\|y_0\|_1&=& \sum_{j\in T}|x_{0j}+(y_{0j}-x_{0j})|+\sum_{j\in T} |y_{0j}|\\
&\ge & \sum_{j\in T}\mathrm{sign}(x_{0j}) (x_{0j}+(y_{0j}-x_{0j}))+\sum_{j\not\in T} y_{0j} \left< w, A_j\right>\\
&=& \sum_{j\in T} |x_{0j}|+\sum_{j\in T} (y_{0j}-x_{0j})\left<w, A_j\right> +
\sum_{j\not\in T} y_{0j} \left< w, A_j\right>\\
&=&\sum_{j\in T} |x_{0j}|+\left<w, \sum_{j\in T\cup T^c} y_{0j}A_j-\sum_{j\in T} x_{0j}A_j\right>\\
&=&\|x_0\|_1+\langle w, b-b\rangle \\
&=&\|x_0\|_1
\end{eqnarray*}  

Thus, the inequalities in the above computation must in fact be equality. Since $|\left<w, A_j\right>|$ is strictly less than $1$ for all $j\not \in T$, this in particular forces $y_{0j}=0$ for all $j\not \in T$. Thus
$$\sum_{j\in T} (y_{0j}-x_{0j})A_j=f-f=0.$$

Since all columns in $A_T$ are independent, we must have $y_{0j}=x_{0j}$ for all $j\in T$. Thus $x_0=y_0$. This concludes the proof of our theorem. 

\end{proof}

The above theorem points out the necessary and sufficient condition that in the noise-free setting $\mathcal{P}_1$ exactly recover the sparse signal $x_0$.
The necessity and sufficiency comes from the KKT condition in convex optimization theory~\citep{candes2005}. However this condition is difficult to check due to the presence of $w$. If we further assume that $w$ lies in the column span of $A_T$,  the condition in Proposition \ref{pro::P1} reduces to the following condition.

\textbf{Irrepresentable Condition (IRR)} The matrix $A$ satisfies the IRR condition with respect to $T=\mathrm{supp}(x_0)$, if $A^*_TA_T$ is invertible and
\begin{equation*}
\|A^*_{T^c}A_T(A^*_TA_T)^{-1}\|_\infty<1,
\end{equation*}
or, equivalently,
\begin{equation*}
\|(A^*_TA_T)^{-1}A^{*}_TA_{T^c}\|_1<1,
\end{equation*}

where $\|\cdot \|_\infty$ stands for the matrix sup-norm, i.e., $\|A\|_\infty := \max_i \sum_j |A_{ij}|$ and $\|A\|_{1} = \max_{j}\sum_{i} |A_{ij}|$.

\begin{proposition}
By restricting that $w$ lies in the image of $A_T$, the conditions in proposition \ref{pro::P1} reduce to the IRR condition. 
\end{proposition}

\begin{proof}
Since $w$ lies in the image of $A_T$, we can write  $w=A_Tv$. To make sure that the first condition in Proposition \ref{pro::P1} holds, we must have 
$v=(A^*_TA_T)^{-1}\mathrm{sign}(x_0)$, so 
$$w=A_T(A_T^*A_T)^{-1} \mathrm{sign}(x_0).$$
Now the second condition in proposition \ref{pro::P1} can be equivalently written as $$\|A_{T^c}^* A_T(A^*_TA_T)^{-1}\|_\infty <1,$$
which is exactly the IRR condition. 
\end{proof}

Intuitively, the IRR condition requires that, for the true sparsity signal $x_{0}$,  the relevant bases $A_{T}$ is not highly correlated with irrelevant bases $A^{c}_{T}$. Note that this condition only depends on $A$ and $x_0$, which is easier to check. The assumption that $w$  lies in the column span of $A_T$ is mild; it is actually a necessary condition so that $x_0$ can be reconstructed by Lasso~\citep{tibshirani1996} or Dantzig selector~\citep{candes2007}, even under Gaussian-like noise assumptions~\citep{zhao2006,yuan2007}.


\subsection{Detecting Cliques of Equal Size}
In this subsection, we present sufficient conditions of IRR which can be easily verified. We consider the case that $A = R^{j,k}$ with $j<k$. Given data $b$ about
all $j$-sets, we want to infer important $k$-cliques.
Suppose $x_0$ is a sparse signal on all $k$-cliques. We have the following theorem, which is a direct result of Lemma \ref{lemma:worstcase}.

\begin{theorem}\label{thm:worstcase}
Let $T=\mathrm{supp}(x_0)$, if we enforce the overlaps among $k$-cliques in $T$ to be
no larger than $r$, then $r\leq j-2$  guarantees the IRR condition. 
\end{theorem}


\begin{lemma}\label{lemma:worstcase}
Let $T=\mathrm{supp}(x_0)$ and $j\ge 2$. Suppose for any $\sigma_1, \sigma_2\in T$, the two cliques corresponding to $\sigma_{1}$ and $\sigma_{2}$ have overlaps no larger than $r$, we have
\begin{enumerate}
\item If $r\leq j-2$, then $\|A^*_{T^c}A_T(A^*_TA_T)^{-1}\|_\infty<1$;
\item If $r=j-1$, then $\|A^*_{T^c}A_T(A^*_TA_T)^{-1}\|_\infty\le 1$ where equality holds with certain examples; 
\item If $r=j$, there are examples such that $\|A^*_{T^c}A_T(A^*_TA_T)^{-1}\|_\infty>1$.
\end{enumerate}
\end{lemma}

One thing to note is that Theorem \ref{thm:worstcase} is only an easy-to-verify condition based on the worst-case analysis, which is sufficient but not necessary. In fact, what really matters is the IRR condition.
It uses a simple characterization of allowed clique overlaps which guarantees the IRR Condition. Specifically, 
clique overlaps no larger than $j-2$ is sufficient to guarantee the exact sparse recovery by $\mathcal{P}_1$, while larger overlaps may violate the IRR Condition. 
Since this theorem is based on a worst-case analysis, in real applications, one may encounter examples which have overlaps larger than $j-2$ while $\mathcal{P}_1$ still works.

In summary, IRR is sufficient and almost necessary to guarantee exact recovery. Theorem \ref{thm:worstcase} tells us the intuition behind the IRR is that {\em overlaps among
cliques must be small enough}, which is  easier to check. In the next subsection, we show that IRR is also sufficient to guarantee stable recovery with noises.

\begin{proof}
To prove Lemma \ref{lemma:worstcase}, given any $\tau\in T^c$, we define 
$$\mu_\tau\equiv\sum_{\sigma\in T}\frac{{|\tau\cap \sigma|\choose j}}{{k\choose j}}. $$
the intuition of such a definition is that 
\begin{eqnarray}
\sup_{\tau\in T^c} \mu_\tau=\|A^*_{T^c} A_T\|_\infty.
\end{eqnarray} 
As we will see in the following proofs, we essentially try to bound $\mu_\tau$ for $\tau\in T^c$.

Before we present the detailed technical proof, we first introduce the high-level idea: our main purpose is to bound $\|A^*_{T^c}A_T(A^*_TA_T)^{-1}\|_\infty$. Since each entry of the matrix $ A^*_TA_T$ is indexed by two $k$-sets, the value of this entry represents how many $j$-sets are contained in the intersection of these two $k$-sets. Under the condition that  $r\leq j-1$, it's straightforward that the matrix $ A^*_TA_T$ is an identity. Therefore, bounding  $\|A^*_{T^c}A_T(A^*_TA_T)^{-1}\|_\infty$ is equivalent as bounding $\|A^*_{T^c}A_T\|_\infty$, which is exactly $\sup_{\tau\in T^c} \mu_\tau$.

{\em Proof of the case under Condition 1}

Under Condition 1, since any $\sigma_1, \sigma_2\in T$ satisfy $|\sigma_1\cap \sigma_2|
\le j-2$, hence any two columns in $T$ are orthogonal. 
This implies $A^*_TA_T$ is an identity matrix.  

Now given $\tau\in T^c$, we will prove $\mu_\tau<1$ under condition 1. If this
is true, then  $$\sup_{\tau\in T^c} \mu_\tau=\|A^*_{T^c} A_T\|_\infty=
\|A^*_{T^c}A_T (A^*_TA_T)^{-1}\|_\infty<1 $$

Let $T=\{\sigma_1, \sigma_2, \cdots, \sigma_{|T|}\}$ 
where $\sigma_i$($1\le i\le |T|$) are $k$-sets. 
We need to prove $$\mu_\tau=\sum_{i=1}^{|T|}\frac{{|\tau\cap \sigma_i|\choose j}}{{k\choose j}}<1$$ for all $\tau \in T^{c}$.

Let $\mathcal{M}_i=\{\rho: |\rho|=j, \rho\subset \tau\cap\sigma_i \}$, so $\mathcal{M}_i$ is a collection of $j$-sets
of $\tau\cap \sigma_i$ (Here if $|\tau\cap\sigma_i|<j$, then $\mathcal{M}_i$ is simply an empty set).
Obviously, we have $|\mathcal{M}_i|={|\tau\cap \sigma_i|\choose j}$. 
So $$\sum_{i=1}^{|T|} {|\tau\cap\sigma_i|\choose j}=\sum_{i=1}^{|T|} |\mathcal{M}_i|.$$

Now we note the fact that for any 
$1\le i, l\le |T|$, we have $\mathcal{M}_i\cap \mathcal{M}_l=\emptyset$. This is true because otherwise
suppose $\rho\in \mathcal{M}_1\cap \mathcal{M}_2$, then this mean $\rho$ is a $j$-set of $\mathcal{M}_1$ and $\mathcal{M}_2$. Hence
$\rho\subset \tau\cap\sigma_1, \rho\subset \tau\cap\sigma_2$, which implies that 
$$|\sigma_1\cap\sigma_2|\ge |(\tau\cap\sigma_1)\cap (\tau\cap\sigma_2)|\ge |\rho|\ge j.$$
This contradicts with the condition that $\sigma_i$'s($1\le i\le T$) have overlaps at most $j-2$. 
So $\mathcal{M}_i$ must be pairwise disjoint. Hence
$$\sum_{i=1}^{|T|} {|\tau\cap\sigma_i|\choose j}=\sum_{i=1}^{|T|} |\mathcal{M}_i|=|\cup_{i=1}^{|T|} \mathcal{M}_i|$$

For any $1\le i\le |T|$, every $\rho\in \mathcal{M}_i$ is a $j$-set of $\tau\cap\sigma_i$. Hence
$\rho$ is of course a $j$-set of $\tau$. The set $\tau $ is of size $k$. So if we let 
$\mathcal{M}_0=\{\rho: |\rho|=j, \rho\subset \tau\}$ which is the collection of all $j$-sets of $\tau$, 
then we have $\cup_{i=1}^{|T|} \mathcal{M}_i\subset \mathcal{M}_0$. So 
$|\cup_{i=1}^{|T|} \mathcal{M}_i|\le  |\mathcal{M}_0|\le {k\choose j}$. 

Till now, we actually proved $\mu_\tau\le 1$. All the above proof about $\mu_\tau\le 1$ for any 
$\tau\in T^c$ will remain valid 
for condition 2. In the next, 
we prove if any $\sigma_i, \sigma_l\in T$ satisfy $|\sigma_i\cap\sigma_l|\le j-2$, 
then equality can not hold. 

Without loss of generality, we assume $|\sigma_1\cap\tau|\ge j$, otherwise if none of $\sigma_i$'s 
satisfies $|\sigma_i\cap\tau|\ge j$, then $\mu_\tau=0$ which actually finishes the proof. To show the the equality will not hold, we only need to find one $j$-set that is does not belong to $\cup_{i}{\mathcal{M}_{0}}$.

In this case, we can let $\tau=\{1,2,\cdots,k\}$, $\sigma_1=\{1,2,\cdots,s, k+1, k+2, 2k-s\}$ 
where $j\le s\le k-1$($s\le k-1$ because otherwise $\sigma_1=\tau$ which contradicts with the fact
that $\sigma_1\in T, \tau\in T^c$). 
Now we show that $\rho_0=\{1,2, \cdots, j-1, s+1\}$ is not a member of $\cup_{i=1}
^{|T|} \mathcal{M}_i$. 
Clearly $\rho_0$ is not a member of $\mathcal{M}_1$ because $s+1\not\in \sigma_1$. 
Now it remains to show that $\rho_0$ is not
a member of any $\mathcal{M}_i$($2\le i\le |T|$). If this was not true, say $\rho_0\in \mathcal{M}_2$, then 
$\rho_0\subset (\tau\cap\sigma_2)\subset \sigma_2$, then $\{1, 2, \cdots, j-1\}\subset \sigma_1\cap
\sigma_2$, which contradicts with the condition that $|\sigma_1\cap\sigma_2|\le j-2$. 

While it is clear that $\rho_0\ in \mathcal{M}_0$, so this means $\cup_{i=1}^{|T|}\mathcal{M}_i$ is a proper subset of 
$\mathcal{M}_0$. So $|\cup_{i=1}^{|T|} \mathcal{M}_i|< {k\choose j}$ which means $\mu_\tau<1$.

{\em Proof of the case under Condition 2}

Under condition 2, then almost the same as proof for lemma 1. We have 
$A^*_TA_T$ is an identity matrix and $\mu_\tau\le 1$. However, one can not show $\mu_\tau<1$ in
this case. We have the following example where if $n$ is large enough, then 
$\mu_\tau$ can happens to be equal to one exactly. 

Let $\tau=\{1,2,\cdots,k\}\in T^c$. Denote all the $j$-sets of $\tau$
to be $\rho_1, \rho_2, \cdots, \rho_{{k\choose j}}$. when $n$ is large enough, we choose $k \choose j$ disjoint $(k-j)$-sets of $\{k+1, k+2, \cdots, n\}$, denoted by $\omega_1, \omega_2, \cdots, \omega_
{{k\choose j}}$. 

Let $T=\{\sigma_1, \sigma_2, \cdots, \sigma_{|T|}\}$, where $\sigma_i=\rho_i\cup\omega_i$. Hence $|T|={k\choose j}$ and  
$\sigma_i$'s satisfy $|\sigma_i\cap\sigma_j|\le j-1$. But 
$$\sum_{i=1}^{|T|}\frac{{|\tau\cap\sigma_i|\choose j}}{{k\choose j}}=\sum_{i=1}^{|T|}\frac{1}{{k\choose j}}=1.$$

{\em Proof of the case under Condition 3}

Under condition 3, we can construct examples where $$\|A^*_{T^c}A_T(A^*_TA_T)^{-1}\|_\infty>1.$$ 
Let $\rho_1, \rho_2, \cdots, \rho_{{k\choose j}}$ be all $j$-sets of $\{1,2,\cdots, k\}$. For large enough $n$, it is possible to choose ${k\choose j }+1$ disjoint $(k-j)$-sets of $\{k+1, k+2, \cdots, n\}$, say
$\omega_0, \omega_1, \omega_2, \cdots, \omega_
{{k\choose j}}$. Let $\sigma_i=\rho_i\cup \omega_i$ for $1\leq i \leq {k\choose j}$ and $\sigma_0=\rho_1\cup \omega_{0}$. Define $T= \{\sigma_0, \sigma_1, \sigma_2, \cdots, \sigma_{{k\choose j}}\}$ which is of size $|T|={k\choose j}+1$. 

In this case, $|\sigma_i\cap \sigma_l|=j-1$ for any $1\le i, l\le {k\choose j}$ and $|\sigma_0\cap \sigma_1|=j$, 
$|\sigma_0\cap \sigma_i|\le j-1$ for any $2\le i\le {k\choose j}$. Then 
$A^*_TA_T$ is a ${k\choose j}+1$ by ${k\choose j}+1$ matrix shown below with rows and columns corresponds to 
$\{\sigma_0, \sigma_1, \cdots, \sigma_{{k\choose j}}\}$
\[ A^*_TA_T=\left[ \begin{array}{ccccccc}
1 & \epsilon & \vline & 0 & 0 & \cdots & 0\\
\epsilon & 1 & \vline & 0 & 0 & \cdots & 0\\\hline
0 & 0 & \vline & 1 & 0 & \cdots & 0\\
0 & 0 & \vline & 0 & 1 & \cdots & 0\\
0 & 0 & \vline & \vdots & \vdots & \ddots & 0\\
0 & 0 & \vline & 0 & 0 & \cdots & 1\\
\end{array} \right]\] 
Here $\epsilon=\frac{1}{{k\choose j}}$. The inverse of the matrix is
\[ (A^*_TA_T)^{-1}= \left[ \begin{array}{ccccccc}
\frac{1}{1-\epsilon^2} & -\frac{\epsilon}{1-\epsilon^2} & \vline & 0 & 0 & \cdots & 0\\
-\frac{\epsilon}{1-\epsilon^2} & \frac{1}{1-\epsilon^2} & \vline & 0 & 0 & \cdots & 0\\\hline
0 & 0 & \vline & 1 & 0 & \cdots & 0\\
0 & 0 & \vline & 0 & 1 & \cdots & 0\\
0 & 0 & \vline & \vdots & \vdots & \ddots & 0\\
0 & 0 & \vline & 0 & 0 & \cdots & 1\\
\end{array} \right]\] 

Consider $\tau=\{1,2,\cdots, k\}\in T^c$, then the row corresponds to $\tau$ for $A^*_{T^c}A_T$
is a vector of length $|T|={k\choose j}+1$ with each entry being $\epsilon=\frac{1}{{k\choose j}}$.
So the row vector corresponds to $\tau$ in $A^*_{T^c}A_T(A^*_TA_T)^{-1}$ is a vector of
length ${k\choose j}+1$, 
$[\frac{\epsilon}{1+\epsilon}, \frac{\epsilon}{1+\epsilon}, \epsilon, \epsilon, \cdots, \epsilon]$. 
This vector has row sum 
$$\frac{2\epsilon}{1+\epsilon}+({k\choose j}-1)\epsilon=\frac{2\epsilon}{1+\epsilon}+(
\frac{1}{\epsilon}-1)\epsilon=\frac{1+2\epsilon-\epsilon^2}{1+\epsilon}>\frac{1+2\epsilon-\epsilon}{1+\epsilon}=1$$
Hence in this example $\|A^*_{T^c}A_T (A^*_TA_T)^{-1}\|_\infty>1$. 
\end{proof}

In the following, we construct explicit conditions which allow large overlaps while the IRR still holds, as long as such heavy overlaps do not occur too often among the cliques in $T$. 
The existence of a partition of $T$ in the next theorem is a reasonable assumption in the network settings where network hierarchies exist. 
In social networks, it has been observed by~\cite{girvan2002} that communities themselves also join together to form meta-communities. The assumptions that we made in the next theorem 
where we allow relatively larger overlaps between communities from the same meta-community, while we allow relatively smaller overlaps between communities from different meta-communities 
characterize such a scenario.

\begin{theorem}
Assume $(k+1)/2\le j<k$.  let $T=\mathrm{supp}(x_0)$. Suppose there exist a partition $T=T_1\cup T_2\cup\cdots \cup T_m$ with each $T_i$ satisfies $|T_i|\leq K$, such that 
\begin{itemize}
\item for any $\sigma_i,\sigma_j$ belong to the same partition, $|\sigma_i\cap\sigma_j|\leq r$;
\item for any $\sigma_i,\sigma_j$ belong to different partitions, $|\sigma_i\cap\sigma_j|\leq 2j-k-1$. 
\end{itemize}
If $K$ satisfies $$(K-1){r\choose j}/{k\choose j}< 1/4, \ \ \ \ \ 
\Bigg({k-1\choose j}+(K-1){(k+r)/2\choose j}\Bigg)/{k\choose j}\leq 3/4, $$
then IRR holds. 
\end{theorem}

\begin{proof} We will show the following two inequalities hold. 
$$\|A_T^*A_T-I\|_\infty\leq (K-1){r\choose j}/{k\choose j}, \ \ \ \ \ 
\|A_{T^c}^*A_T\|_\infty \leq \Bigg({k-1\choose j}+(K-1){(k+r)/2\choose j}\Bigg)/{k\choose j}. $$

We first bound the sup-norm of $A_T^*A_T-I$. Note that when $\sigma_i$ and $\sigma_j$ belong to different partitions of $T$, then $|\sigma_i\cap \sigma_j|=0$ because their overlap is no larger than $2j-k-1$ which is strictly smaller than $j$. So $A_T^*A_T$ is a block diagonal matrix with block sizes $|T_1|$, $|T_2|$, $\cdots$, $|T_m|$, and each diagonal entry of $A_T^*A_T$ is one.

Thus, for any $\sigma\in T$, only cliques from the same partition as $\sigma$
may have overlaps with $\sigma$ greater than $j$. Thus, the row sum of 
$A_T^*A_T-I$ can be bounded by $(K-1){r\choose j}/{k\choose j}$. So the first inequality is now established. 

To prove the second inequality, we observe that for a fixed $\tau\in T^c$, $|\tau\cap \sigma_i|\ge j$ and $|\tau \cap \sigma_j|\ge j$ can not hold
at the same time for any $\sigma_i$ and $\sigma_j$ belong to different partitions. 
This is because otherwise, we will have 
\begin{eqnarray*}
|\tau|&\ge & |\tau\cap (\sigma_i\cup\sigma_j)|= |\tau\cap \sigma_i|+|\tau\cap \sigma_j|-|\tau\cap\sigma_i\cap\sigma_j|\\
&\ge& j+j-(2j-k-1)=k+1
\end{eqnarray*}
Thus, all $\sigma$'s which have intersections with a fixed $\tau $ no less than $j$ must lie in the same partition of $T$. 

For the same reason, we can show that for a fixed $\tau\in T^c$, 
$|\tau\cap\sigma_i|\ge (k+r+1)/2$ and $|\tau\cap\sigma_j|\ge (k+r+1)/2$ can not hold at the same time for $\sigma_i$ and $\sigma_j$ belong to the same
partition of $T$. This is because otherwise, we will have 
\begin{eqnarray*}
|\tau|&\ge & |\tau\cap (\sigma_i\cup\sigma_j)|= |\tau\cap \sigma_i|+|\tau\cap \sigma_j|-|\tau\cap\sigma_i\cap\sigma_j|\\
&\ge& (k+r+1)/2+(k+r+1)/2-r=k+1
\end{eqnarray*}

Thus we know the maximum row sum of $A_{T^c}^*A_T$ is bounded from above by
$$\Bigg({k-1\choose j}+(K-1){(k+r)/2\choose j}\Bigg)/{k\choose j}.$$

Now if $K$ further satisfies 
$$(K-1){r\choose j}/{k\choose j}< 1/4, \ \ \ \ \ 
\Bigg({k-1\choose j}+(K-1){(k+r)/2\choose j}\Bigg)/{k\choose j}\leq 3/4. $$

then, we have 
$$\|A_T^*A_T-I\|_\infty< 1/4, \ \ \ \ \ 
\|A_{T^c}^*A_T\|_\infty \leq  3/4.$$
Thus, 
\begin{eqnarray*}
&&\|A_{T^c}^*A_T(A_T^*A_T)^{-1}\|_\infty\\
&\leq & \|A_{T^c}^*A_T\|_\infty \|(A_T^*A_T)^{-1}\|_\infty\\
&\leq & \|A_{T^c}^*A_T\|_\infty (1+\sum_{i=1}^\infty \|(A_T^*A_T-I)\|_\infty^i)\\
&<& 3/4(1+\sum_{i=1}^\infty (1/4)^i)=1
\end{eqnarray*}
So IRR holds under our conditions. 
\end{proof}

The basis matrix $A=R^{j,k}$ have ${n\choose k}$ bases, which is not polynomial with respect to $k$. As we will see from later sections, a 
practical implementation of the Radon basis pursuit for the clique detection problem works on a subset of bases among all ${n\choose k}$ bases. In that case, we are actually solving 
$\mathcal{P}_1$ and $\mathcal{P}_{1,\delta}$ with the basis matrix $\bar{A}$, which is only a submatrix of $A$ with a subset of column bases extracted. We have the following theorem regarding 
this scenario.

\begin{theorem} Denote the set of all cliques for columns in $\bar{A}$ by $S$, where $\bar{A}$ is a submatrix of $A$. Assume any two $k$-cliques in $\bar{A}$ have intersections at most $r$, i.e. $\forall \sigma_i,\sigma_j \in T\cup T^c$, $|\sigma_i \cap \sigma_j |\leq r$, where $T=\mathrm{supp}(x_0)\subset S$, and $T^c$ is the complement of $T$ with respect to $S$. Then IRR holds if
\begin{equation}\label{eq:condition_r}
r \leq \left( \frac{1}{|T| (1+\sqrt{|T|})}\right)^{1/j} k 
\end{equation}
\end{theorem}

\begin{proof} Note that 
\begin{eqnarray*}
\|\bar{A}_{T^c}^* \bar{A}_T (\bar{A}_T^*\bar{A}_T)^{-1} \|_\infty & \leq & \|\bar{A}_{T^c}^* \bar{A}_T \|_\infty \|(\bar{A}_T^* \bar{A}_T)^{-1}\|_\infty  \\
& \leq & \|\bar{A}_{T^c}^* \bar{A}_T \|_\infty \cdot \sqrt{|T|} \|(\bar{A}_T^* \bar{A}_T)^{-1}\|_2 
\end{eqnarray*}
So it suffices to show 
\[  \|\bar{A}_{T^c}^* \bar{A}_T \|_\infty \cdot \sqrt{|T|} \|(\bar{A}_T^* \bar{A}_T)^{-1}\|_2  < 1\]
under condition (\ref{eq:condition_r}).

Firstly, 
\begin{eqnarray*}
\|\bar{A}^*_{T^c} \bar{A}_T\|_\infty & = & \max_{\tau \in T^c} \sum_{\sigma\in T}\frac{{|\tau\cap \sigma|\choose j}}{{k\choose j}}  \leq  \frac{|T|{r\choose j}}{{ k\choose j} }, \ \ \ \mbox{since $|\tau\cap \sigma|\leq r$}.
\end{eqnarray*}
At least we need 
\begin{equation} \label{eq:c1}
|T|{r\choose j} / { k\choose j}  <1.
\end{equation} 

Secondly, let $K = \bar{A}_T^* \bar{A}_T$, then
$$ K_{ii} = 1 $$
and since $\forall \sigma_i,\sigma_j \in T$, $|\sigma_i \cap \sigma_j|\leq r$, we have 
$$ K_{ij} \leq \frac{{r \choose j}} {{k \choose j}}. $$
Under condition (\ref{eq:c1}), $K$ is diagonal dominant, i.e.
$$ K_{ii} > \sum_{j\neq i} |K_{ij}|. $$
Then by Girshgorin Circle Theorem, 
$$  \lambda_{\min} \geq 1 - \sum_{j\neq i}|K_{ij}| \geq 1 - (|T|-1) {r \choose j}/{k \choose j}\geq 1 - |T|{r \choose j}/{k \choose j}. $$

Therefore it suffices to have 
\[ \frac{|T|{r\choose j}} {{ k\choose j}}  \frac{\sqrt{|T|}}{ 1 - |T|{r \choose j}/{k \choose j}} <1 \]
which gives 
\[ {r \choose j} < \frac{1}{ |T|(1+\sqrt{|T|} )} {k \choose j}. \]
To satisfy this, it suffices to assume 
\[ r < \left(\frac{1}{ |T|(1+\sqrt{|T|} )} \right)^{1/j} k. \]

\end{proof}

\subsection{Stable Recovery Theorems}

In applications, one always encounters examples with noise such that exact sparse recovery is impossible.
In this setting, $\mathcal{P}_{1, \delta}$ will be a good replacement of $\mathcal{P}_1$ as a robust
reconstruction program. Here we present stable recovery theorem of $\mathcal{P}_{1,\delta}$ with bounded noise.

\begin{theorem}
Under the general framework \eqref{eq::general}, we assume that $\|z\|_\infty \leq \epsilon$, $|T|= s$, and the IRR
$$\|A^*_{T^c} A_T (A^*_T A_T)^{-1} \|_\infty \leq \alpha\le 1/s.$$
Then the following error bound holds for any solution $\hat{x}_\delta$ of
$\mathcal{P}_{1,\delta}$,
\begin{equation}\label{eqn:stablerecovery}
\|\hat{x}_\delta - x_0 \|_1 \leq  \frac{2s (\epsilon + \delta)} {1 - \alpha s} \|A_T(A^*_TA_T)^{-1}\|_1.
\end{equation}
\end{theorem}

\begin{proof}
Let $h = \hat{x}_\delta - x_{0}$.  Note that $\|A \hat{x}_\delta -b\|_\infty \leq \delta$ and $z=Ax_{0}-b$ with $\|z\|_\infty \leq \epsilon$. Then
\begin{equation} \label{eq:Ah} 
\|Ah\|_\infty = \|A\hat{x}_\delta - A x_{0} \|_\infty = \|A\hat{x}_\delta - b+ b - Ax_{0}\|_\infty\leq \|A\hat{x}_\delta-b\|_\infty + \|z\|_\infty \leq \delta + \epsilon. 
\end{equation}
We denote $\hat{x}_\delta |_T $ as constraining $\hat{x}_{\delta}$ on the support $T$, i.e. all the entries of $\hat{x}_{\delta}$ corresponding to $T^{c}$ will be set to zero. From the optimization problem in ($\mathcal{P}_{1,\delta}$), we know that $\|x_{0}\|_1 \geq \|\hat{x}_{\delta}\|_1$,
\begin{equation}\label{eq:ht} 
\|h_{T}\|_1 = \|x_{0}-\hat{x}_\delta |_T  \|_1 \geq \|x_{0}\|_1 - \|\hat{x}_\delta|_T \|_1 \geq \|\hat{x}_\delta \|_1 - \|\hat{x}_\delta|_T \|_1 = \| \hat{x}_\delta|_{T^c} \|_1 = \|h_{T^c} \|_1 .
\end{equation}
Therefore, 
\begin{eqnarray*}
&&|\langle A h , A_T (A^*_T A_T)^{-1} h_T\rangle| \\
& =  & |\langle A_T h_T , A_T(A^*_T A_T)^{-1} h _T \rangle + \langle A_{T^c} h_{T^c},A_T(A^*_T A_T)^{-1} h _T \rangle | \\
& \geq & \|h_T\|_2^2 - |\langle h_{T^c} ,    A^*_{T^c}A_T(A^*_T A_T)^{-1} h _T \rangle | \\
& \geq & \|h_T\|_2^2 - \|h_{T^c}\|_1 \|A^*_{T^c}A_T(A^*_T A_T)^{-1} h _T \|_\infty \\
& \geq & \frac{1}{s} \|h_T\|_1^2 - \alpha \|h_{T^c} \|_1 \|h_T\|_\infty \\
& \geq & \frac{1}{s} \|h_T\|_1^2 - \alpha \|h_{T^c}\|_1 \|h_T\|_1 \\
& \geq & \left( \frac{1}{s} - \alpha \right) \|h_T\|_1^2
\end{eqnarray*}
where the last step is due to $\|h_T\|_1 \geq \|h_{T^c} \|_1$ in the inequality (\ref{eq:ht}). On the other hand,
\begin{eqnarray*}
&&|\langle A h , A_T (A^*_T A_T)^{-1} h_T\rangle| \\
& \leq  & \|A h\|_\infty \|A_T(A^*_T A_T)^{-1} h_T \|_1  \\
& \leq &  (\delta +\epsilon) \|A_T(A^*_T A_T)^{-1} \|_1 \|h_T\|_1
\end{eqnarray*} 
using (\ref{eq:Ah}). Combining these two inequalities yields
\[   \|h_T\|_1 \leq \frac{s(\delta+\epsilon)}{1-\alpha s} \|A_T(A^*_T A_T)^{-1} \|_1, \]
as desired. 

\end{proof}

In the special case where $k=j+1$, we have:
\begin{corollary}
Let $k=j+1, |T|= s$, and for any $\sigma_1, \sigma_2\in T$, the two cliques corresponding to $\sigma_{1}$ and $\sigma_{2}$ have overlaps no larger than $r$. Then we have
$\|A^*_{T^c}A_T (A^*_TA_T)^{-1}\|_\infty \le 1/(j+1)$, and thus the following error bound for solution $\hat{x}_\delta$ of $\mathcal{P}_{1,\delta}$ holds:
\[ \|\hat{x}_\delta - x_0 \|_1 \leq  \frac{2s (\epsilon + \delta)} {1 - \frac{s}{j+1}} \sqrt{j+1}, \ \ \ \ \ \ s<j+1. \]
\end{corollary}

\begin{proof}
This corollary follows follows from the Lemma above. Note that when the conditions in Theorem 2 hold, $A^*_T A_T = I$ and $\|A_T\|_1 \leq \sqrt{{k \choose j}} = \sqrt{j+1}$. 

Now it suffice to establish the fact that in this special case, we have
$$\|A^*_{T^c}A_T(A^*_TA_T)^{-1}\|_\infty\le \frac{1}{j+1}<1$$
Note that since any $\sigma_1, \sigma_2\in T$ satisfy $|\sigma_1\cap\sigma_2|\le j-2$, 
we have $A^*_TA_T$ is an identity matrix. 
So $\|A^*_{T^c}A_T(A^*_TA_T)^{-1}\|_\infty=\|A^*_{T^c}A_T\|_\infty$. 
Now assume $\tau\in T^c$, let $S_{\tau}=\{\sigma: |\sigma\cap \tau|\ge j, \sigma\in T\}$, 
then $|S_{\tau}|\le 1$. This is because otherwise, suppose 
$\{\sigma_1,\sigma_2\}\subset S_{\tau}$ such that $|S_{\tau}|\ge 2$, then we have
\begin{eqnarray*}
|\tau|&\ge& |\tau\cap (\sigma_1\cup \sigma_2)|=|\tau\cap \sigma_1|+|t\cap \sigma_2|-|t\cap \sigma_1\cap\sigma_2|\\
&\ge & j+j-(j-2)=j+2
\end{eqnarray*}
which contradicts with the fact that $\tau$ is a $j+1$-set.
So there exist at most one $\sigma_0\in T$ such that $|\tau\cap \sigma|\ge j$. 
Let $v_\tau$ be the row vector of $A^*_{T^c}A_T$ with row index correspond to $\tau$. 
Then $\|v_\tau\|_\infty\le \frac{{j\choose j}}{{j+1\choose j}}=\frac{1}{j+1}<1$. 
\end{proof}

\subsection{Identifying Cliques with Mixed Sizes}

In general settings, we need to identify high order cliques of mixed sizes, i.e., cliques of sizes $k_1, k_2, \cdots, k_l$ ($k_1<k_2<\cdots<k_l$), based on
the observed data $b$ on all $j$-sets. One  way  to construct the basis matrix $A$ is by
concatenating $R^{j,k}$ with different $k$'s satisfying $k>j$.
We can then solve $\mathcal{P}_1$ and $\mathcal{P}_{1,\delta}$ for exact recovery and stable recovery with this newly concatenated basis matrix $A$. We have the following theorem:

\begin{theorem}
Suppose $x_0$ is a sparse signal on cliques of sizes $k_1, k_2, \cdots, k_\ell  (j \leq k_1<k_2<\cdots<k_\ell \leq k)$ and $b=Ax_0$.
Let $T=\mathrm{supp}(x_0)$. 
\begin{enumerate}
\item If the cliques in $T$ have no overlaps, then they can be identified
by $\mathcal{P}_1$. 
\item Moreover, if the data $b=Ax_0+z$ is contaminated by the noise $z$,  $\mathcal{P}_{1,\delta}$  provides an estimate of
$x_0$ for which the inequality in (\ref{eqn:stablerecovery}) still holds.
\end{enumerate}
\end{theorem}
\begin{proof}
We prove under the condition that any $\sigma_1,\sigma_2\in T$ satisfy 
$|\sigma_1\cap \sigma_2|=0$, then solve $\mathcal{P}_1$ will exactly identify $x_0$. 

For simplicity, given any $\tau\in T^c$, we define 
$$\mu_\tau=\sum_{\sigma\in T} \frac{1}{\sqrt{{|\tau|\choose j}{|\sigma|\choose j}}}{|\tau\cap \sigma|\choose j}$$

Note that the intersection of $\sigma_1$ and $\sigma_2$ is zero implies that $A^*_TA_T=I$, moreover,
given $\tau\in T^c$, the collection of sets
$\{\tau\cap \sigma| \sigma\in T\}$ are disjoint. Note that if there is only one $\sigma_0$ satisfies 
$|\tau\cap \sigma_0|\ge j$, then
$$\mu_\tau=\frac{1}{\sqrt{{|\tau|\choose j}{|\sigma_0|\choose j}}}{|\tau\cap \sigma_0|\choose j}<1,$$
because it is the inner product of two column vectors corresponds to $\tau$ and $\sigma_0$ of $A$, where there are no two columns in $A$ are identical. 

Now suppose there are at least two $\sigma$'s satisfy, $|\tau\cap \sigma|\ge j$, then we have
\begin{eqnarray*}
\mu_\tau&=&\sum_{\sigma\in T} \frac{1}{\sqrt{{|\tau|\choose j}{|\sigma|\choose j}}}{|\tau\cap \sigma|\choose j}\\
&\le& \sum_{\sigma\in T, |\tau\cap \sigma|\ge j} \frac{1}{\sqrt{{|\tau|\choose j}{|\tau\cap\sigma|\choose j}}}{|\tau\cap\sigma|\choose j}\\
&=& \sum_{\sigma \in T, |\tau\cap \sigma|\ge j} \frac{\sqrt{{|\tau\cap \sigma|\choose j}}}{\sqrt{{|\tau|\choose j}}}
\end{eqnarray*}

Since the collection of sets $\{\tau\cap \sigma| \sigma\in T\}$ are disjoint, so if we can prove $$\sqrt{{|\tau\cap \sigma_1|\choose j}}+\sqrt{{|\tau\cap \sigma_2|\choose j|}}<
\sqrt{{|\tau\cap (\sigma_1\cup \sigma_2)|\choose j}},$$ then we know that 
$$\mu_\tau\le \sum_{\sigma\in T, |\tau\cap \sigma|\ge j} \frac{\sqrt{{|\tau\cap \sigma|\choose j}}}{\sqrt{{|\tau|\choose j}}}< \sqrt{{|\tau\cap 
(\cup_{\sigma\in T, |\tau\cap \sigma|\ge j} \sigma)|\choose j}}/\sqrt{{|\tau|\choose j}}\le1$$

Now we only need to prove the following inequality: suppose $j\ge 2$, given $n_1\ge j, n_2\ge j$, 
we need to prove
$\sqrt{{n_1\choose j}}+\sqrt{{n_2\choose j}}<\sqrt{{n_1+n_2\choose j}}$

The case of $j=2$ can be verified directly, while for $j\ge 3$, we square both sides and we now we only need to prove
${n_1\choose j}+{n_2\choose j}+2\sqrt{{n_1\choose j}{n_2\choose j}} <  {n_1+n_2\choose j}$. Since
$${n_1+n_2\choose j}=\sum_{s=0}^{j} {n_1\choose j-s}{n_2\choose s}.$$ So we know we only need to prove
$2\sqrt{{n_1\choose j}{n_2\choose j}}< n_2{n_1\choose j-1}+n_1{n_2\choose j-1}$. Since
$n_2{n_1\choose j-1}+n_1{n_2\choose j-1}\ge 2\sqrt{n_1n_2{n_1\choose j-1}{n_2\choose j-1}}$, so we only need to verify
$n_1{n_1\choose j-1}> {n_1\choose j}$, this can be easily verified by writing out explicitly both sides. 
\end{proof}

The above theorem provides us a sufficient condition to guarantee exact sparse
recovery with concatenated bases and the stable recovery theory is also established.

\section{A Polynomial Time Approximation Algorithm}\label{sec:algorithm}
In practical applications,  we have pairwise interaction data in a network with $n$ nodes and we wish to infer high order cliques up to size $k$.
Directly constructing $A$ by concatenating Radon basis matrices $R^{j,j}, R^{j,j+1}\ldots, R^{j,k}$
and solving $\mathcal{P}_{1,\delta}$ would incur exponential complexity since $A$ has exponentially many columns with respect to $k$. This would be intractable for inferring high order cliques in large networks.
In this section, we describe a polynomial time (with respect to both $n$ and $k$) approximation algorithm for solving $\mathcal{P}_{1,\delta}$.
Recall that the primal and dual programs $\mathcal{P}_{1,\delta}$ and $\mathcal{D}_{1,\delta}$ are:
\begin{eqnarray*}
{\rm (\mathcal{P}_{1,\delta})}&&~\min\|x\|_1 \ ~\text{s.t.}~\|Ax-b\|_\infty \leq \delta\\
{\rm (\mathcal{D}_{1,\delta})}&&~\max-\delta \|\gamma\|_1-b^*\gamma \ ~\text{s.t.}~\|A^*\gamma\|_\infty \leq 1.
\end{eqnarray*}
\begin{proposition}
The problem  ${\rm (\mathcal{D}_{1,\delta})}$ is the dual of ${\rm (\mathcal{P}_{1,\delta})}$.
\end{proposition}

\begin{proof} Consider an alternative form of $\mathcal{P}_{1,\delta}$,
\begin{eqnarray*}
 \min & & 1^T \xi \\
 \mbox{subject to}&  & -\delta\cdot 1 \leq A x - b \leq \delta\cdot 1 \\
& & -\xi \leq x \leq \xi \\
& & \xi \geq 0
\end{eqnarray*}
whose Lagrangian is
\[ L(x,\xi; \gamma,\lambda,\mu) = 1^T\xi - \gamma_+^T(\delta \cdot 1-Ax+b) - \gamma_-^T(Ax - b + \delta\cdot 1) - \lambda_+^T(\xi - x) - \lambda_- ^T(\xi + x) - \mu^T \xi .\]

Here if we assume $A$ is a matrix of size $M$ by $N$, then 
$\gamma_{+}=(\gamma_{+}(1),\ldots,\gamma_+(M))\in \mathbb{R}^{M}_+$,
$\gamma_{-}=(\gamma_{-}(1),\ldots,\gamma_-(M))\in \mathbb{R}^{M}_+$, 
$\lambda_{+} = (\lambda_{+}(1), \ldots, \lambda_{+}(N))^{T}\in\mathbb{R}^{N}_{+}$, 
$\lambda_{-} = (\lambda_{-}(1), \ldots, \lambda_{-}(N))^{T}\in\mathbb{R}^{N}_{+}$, 
$\mu \in \mathbb{R}^{N}_{+}$ are the Lagrange multipliers.

Then the KKT condition gives
\begin{enumerate}
 \item $A^*( \gamma_+ - \gamma_- )+ (\lambda_+ - \lambda_- )= 0$, 
 \item $1 - (\lambda_+ + \lambda_-) - \mu = 0$,
\end{enumerate}
with $\gamma, \lambda, \mu \geq 0$ and $\gamma_+(\tau)\gamma_-(\tau)=\lambda_+(\tau)\lambda_-(\tau)=0$ for all $\tau$. 

Now we can see that the dual function of $1^T\xi$ is 
$$-\delta (\gamma_+^T+\gamma_-^T)\cdot 1-(\gamma_+^T-\gamma_-^T)b,$$  which is $-\delta \|\gamma\|_1-b^*\gamma$, 
while the constraints for $\gamma$ is $\|A^*\gamma\|_\infty\leq 1$. 
\end{proof}

The key of our algorithm is that we use a polynomial number of variables and constraints to approximate both programs,
yielding an approximate solution for $\mathcal{P}_{1, \delta}$. More precisely, we apply a sequential
primal-dual interior point method to solve the relaxed programs:
\begin{eqnarray*}
{\rm (\mathcal{P}_{1,\delta,T})}&&~\min\|x\|_1 \ ~\text{s.t.}~\|A_Tx-b\|_\infty \leq \delta\\
{\rm (\mathcal{D}_{1,\delta,T})}&&~\max-\delta \|\gamma\|_1-b^*\gamma \ ~\text{s.t.}~\|A_T^*\gamma\|_\infty \leq 1.
\end{eqnarray*}
Here $A_T$ is a submatrix of $A$ where we extract a subset of columns $T$.
We approximate the solution to the original programs by solving the above
relaxed programs where we only use polynomially many columns indexed by $T$. In
particular, we want to find an interior point $\gamma$
for $\mathcal{D}_{1,\delta,T}$ which is also feasible for $\mathcal{D}_{1,\delta}$. With this $\gamma$ available, we can use duality gaps to check convergence
because the current dual objective provides a lower bound for $\mathcal{D}_{1,\delta}$ and any interior point for $\mathcal{P}_{1,\delta,T}$ provides an upper bound for $\mathcal{P}_{1,\delta}$.

Let $A_{i}$ be the $i$-th column of $A$. We need to  sequentially update the column set $T$. When we have
a solution $\gamma$ (which is called the approximate analytic center) for the relaxed program $\mathcal{D}_{1,\delta,T}$, we need to find a new column $A_i$ ($i\in T^c$) which is not feasible in $\mathcal{D}_{1,\delta,T}$.
By incorporating $A_i$ into $T$, the feasible region of $\mathcal{D}_{1,\delta,T}$ is reduced to better
approximate  that of $\mathcal{D}_{1,\delta}$. When the current solution $\gamma$ has no violated constraint, i.e.,
$\gamma$ is feasible for $\mathcal{D}_{1,\delta}$, we use interior point methods to find a series of interior points which converge to the solution
of $\mathcal{D}_{1,\delta, T}$. However, we may obtain a new
interior point $\gamma$ which is not feasible for $\mathcal{D}_{1,\delta}$.
We then go back and add violated constraints.
A formal description is provided in Algorithm \ref{algo:accpm}.
\begin{algorithm}
\caption{\label{algo:accpm}Cutting Plane Method for Solving $\mathcal{P}_{1,\delta}$}
\label{alg1}
\begin{algorithmic}
\STATE Initialize $A=I$, $x=b$, $\gamma=(1,1,\cdots,1)^t$.
\WHILE{TRUE}
\IF{$\exists$ $|A^*_i\gamma|> 1$ where $i\in T^c$}
\STATE $T\leftarrow T\cup \{i\}$, formulate new $\mathcal{D}_{1,\delta,T}$ and $\mathcal{P}_{1,\delta,T}$.
\STATE Find new interior points $\gamma$ and $x$ for $\mathcal{D}_{1,\delta, T}$ and $\mathcal{P}_{1,\delta, T}$ respectively.
\ELSIF {the duality gap is small}
\STATE Get the dual solution $\hat{x}$ and stop.
\ELSE
\STATE  {Find a new interior point $\gamma$ for $\mathcal{D}_{1,\delta,T}$, which optimizes the dual objective.}
\ENDIF
\ENDWHILE
\end{algorithmic}
\end{algorithm}

In Algorithm \ref{algo:accpm}, the first IF statement involves a problem of
finding a violated dual constraint for the current relaxed program.
In the special case where
$\gamma$ are dual variables associated with edges,
the problem becomes the {\em maximum edge weight clique problem}, which is known to be NP-hard.
We use a simple greedy heuristic algorithm, which iteratively adds new nodes in order to maximize summation of edge weights
to solve this problem~\citep{george1978}, which runs in $\mathcal{O}(nk^2)$ time and can return a $0.94$-approximate solution in the average case.
Note that, if $\gamma$ is feasible for the dual relaxation problem with no additional violated constraints, then
$0.94\gamma$ must be feasible for $\mathcal{D}_{1,\delta}$ whose objective is discounted by $0.94$. Thus, we will terminate with an $0.94$-approximate solution.

Let $\eta$ be the threshold to check the duality gap. 
Algorithm \ref{algo:accpm} can also be understood as the column generation method~\citep{dantzig1960}, since
adding a new inequality constraint in the dual program adds a variable to the primal program and thus
adds a column to the basis matrix.
For more details of the algorithm, see~\citet{mitchell2003} and~\citet{yebook1997}. Theoretically,
if one is able to find a violated constraint in constant time
and uses interior point methods to locate approximate centers of the
primal-dual feasible
regions, then Algorithm 1 has computational complexity
$\mathcal{O}(M/\eta^2)$, where $M$ is the number
of dual variables~\citep{mitchell2003,yebook1997}.
In our case, $M\asymp n^2$ and find a violated constraint has
complexity $\mathcal{O}(nk^2)$, thus
algorithm \ref{algo:accpm} has complexity $\mathcal{O}(n^3k^2/\eta^2)$. 

Finally, we note that other iterative algorithms, e.g., Bregman iterations, which have
guaranteed convergence rates~\citep{cai2009} can be used to find
solutions of linear program relaxations in our algorithms.
We also note that, in practice, we never need to explicitly construct the matrix $A$  because there are many
 combinatorial structures within the basis matrix to exploit. For example, operations such as evaluating
inner products between the bases can be evaluated efficiently by directly comparing two sets.

\section{Application Examples}\label{sec:application}

In this section, we provide four application examples to illustrate the effectiveness of the proposed framework in this paper.
As we will see, our clique-based model can deal with overlaps between 
cliques which gives us more community structural information compared against using purely 
clustering methods and the state-of-the-art clique percolation method.
In these examples, we use the clique {\em volume} and {\em conductance}, which arguably are the simplest evaluation criteria of clustering quality,
to evaluate different algorithms.
The clique volume is the sum of edge weights inside the clique, while the clique conductance is
the ratio between the number of weights leaving the clique and the clique volume~\citep{leskovec2010}.

More precisely, let $B_{uv}$ be the element on the $u$-th row and $v$-th column of the adjacency matrix $B$.
The {\em conductance} $\phi(S)$ of a set of nodes $S$ is defined as $$\phi(S)=\frac{\sum_{\{(u,v): u\in S, v\notin S\}} 
B_{uv}}{\min(\mbox{Vol}(S), \mbox{Vol}(V\setminus S))}$$
and {\it volume} is $\mbox{Vol}(S)=\sum_{\{u,v \in S\}} B_{uv}.$

\subsection{Basketball Team Detection}

\begin{figure}[htp]
\begin{center}
\begin{tabular}{cc}
\includegraphics[width=0.5\textwidth]{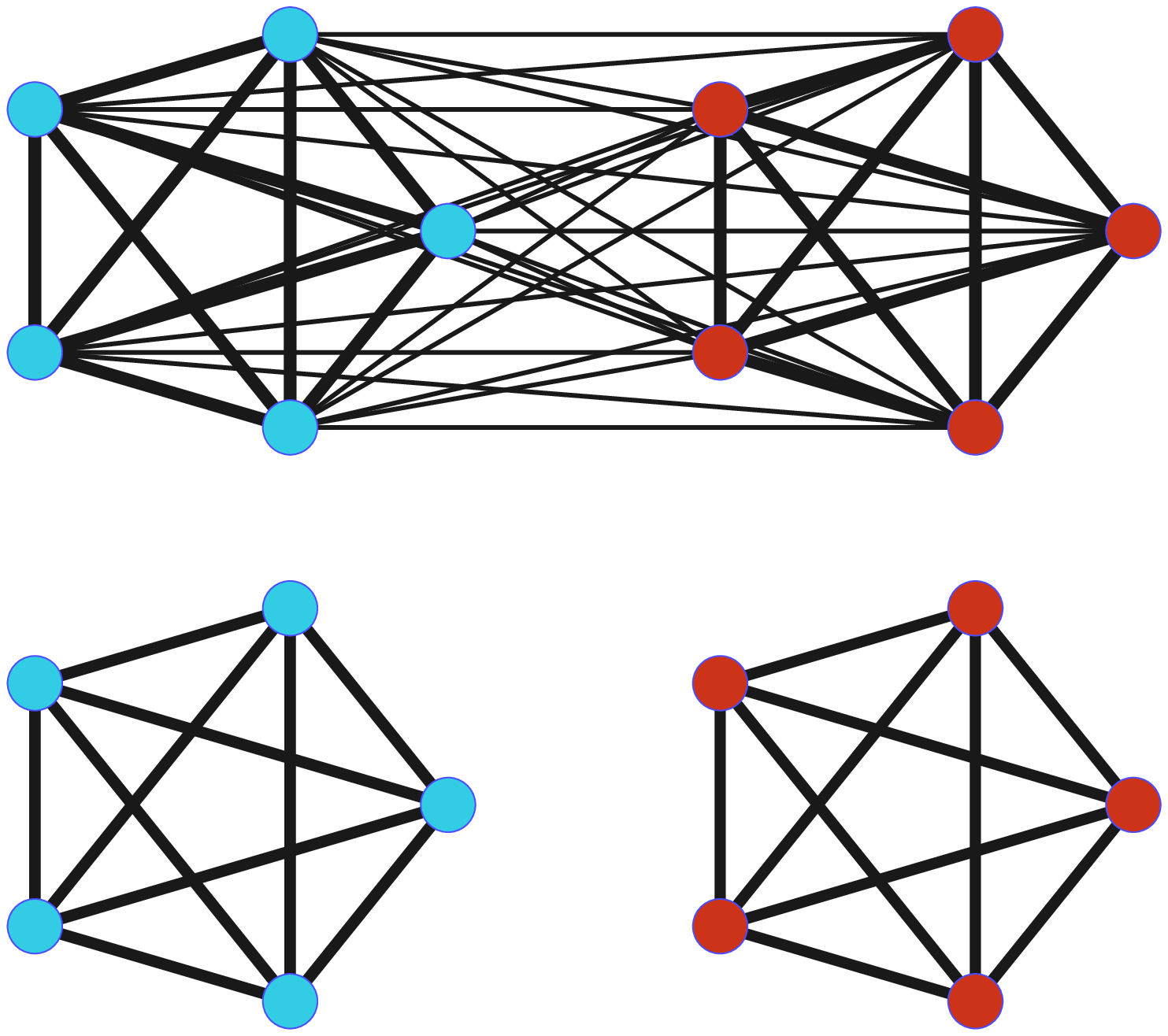} &
\hspace{-0.3in} \includegraphics[width=85mm]{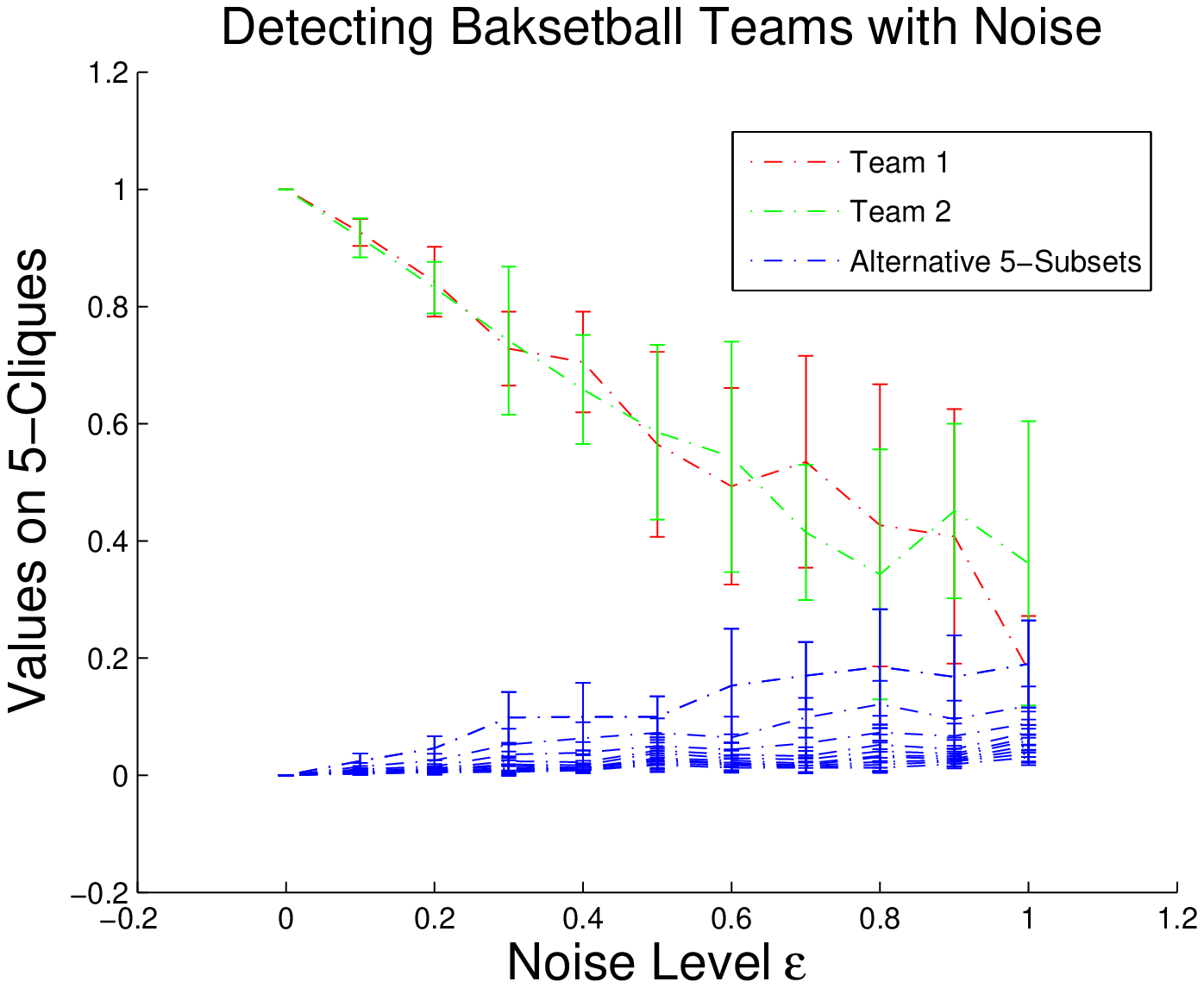} \\
(a) & (b) 
\end{tabular}
\caption{ \label{fig:Basketball}Detecting Basketball Teams with Noise. (a) Two teams in a virtual Basketball Game, with intra-team interaction $1$ and cross-team interaction noise no more than $\epsilon$;  (b) Under a large noise level $\epsilon<0.9$, the two teams are identifiable. For each noise level, we run 100 simulations repeatedly, whose errorbar plot of weights on cliques are shown. }
\end{center}
\end{figure}

Detecting two basketball teams from pairwise interactions among plays is an ideal scenario since the two teams do not overlap. 
Suppose we have $x_0$ which is the true signal indicating the two teams among all $5$-sets of the $10$-player set, i.e., it is sparsely
concentrated on two $5$-sets which correspond to the two teams with magnitudes both equal to one. 
Assume we have observations $b$ of pairwise interactions, i.e. $b=Ax_0+z$, where $z$ is bounded
random noise uniformly distributed in $[-\epsilon, \epsilon]$. We solve $\mathcal{P}_{1,\delta}$, with $\delta=\epsilon$, which is a linear programming search over $x\in \mathbb{R}^{10\choose 5} = \mathbb{R}^{252}$ with a parameter matrix $A \in \mathbb{R}^{{10\choose 2}\times{10 \choose 5}} = \mathbb{R}^{45\times 252}$ and $b\in \mathbb{R}^{45}$.
\begin{figure*}[ht!]
\begin{center}
\begin{tabular}{cc}
\includegraphics[width=80mm]{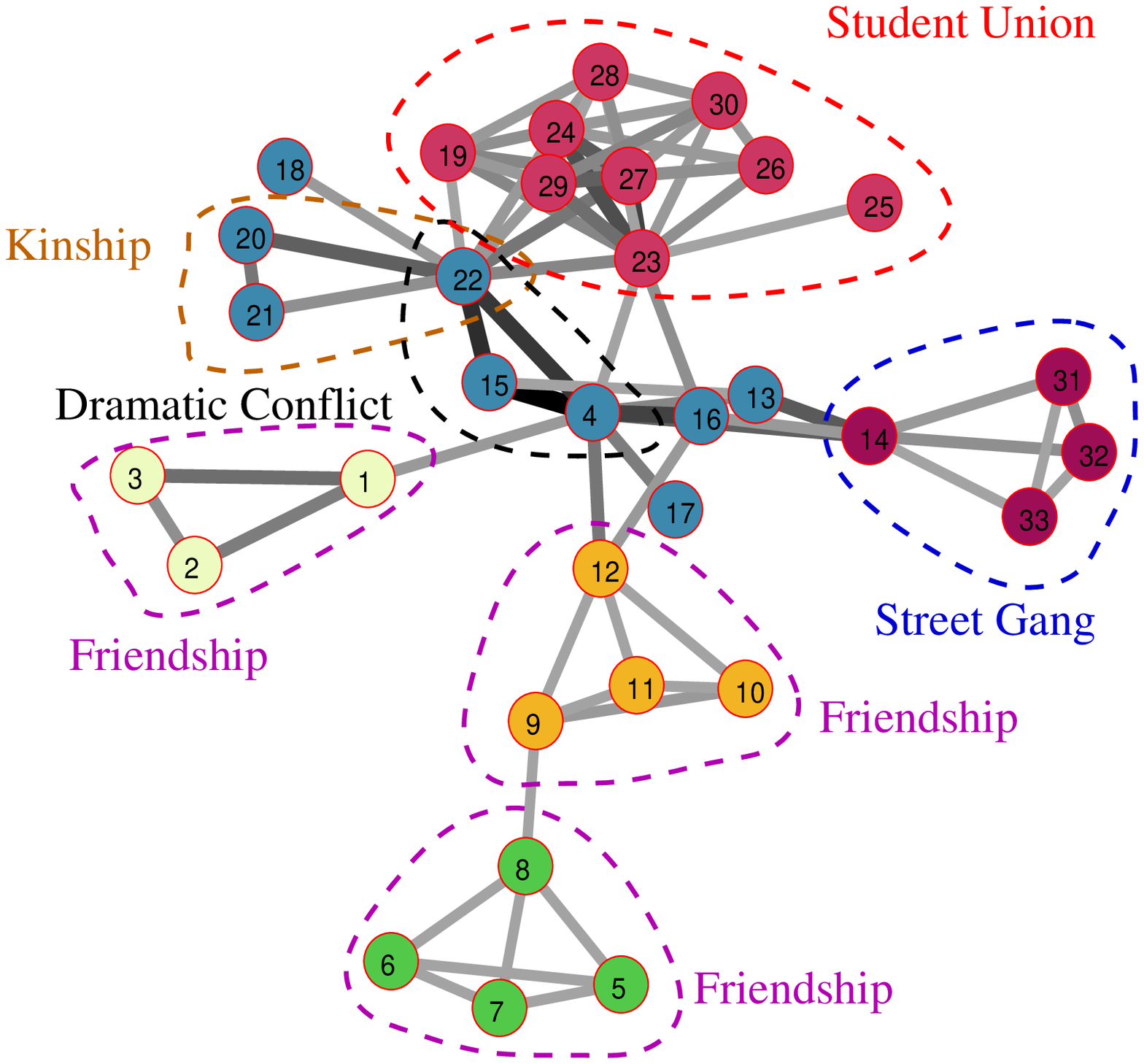} &
\includegraphics[width=80mm]{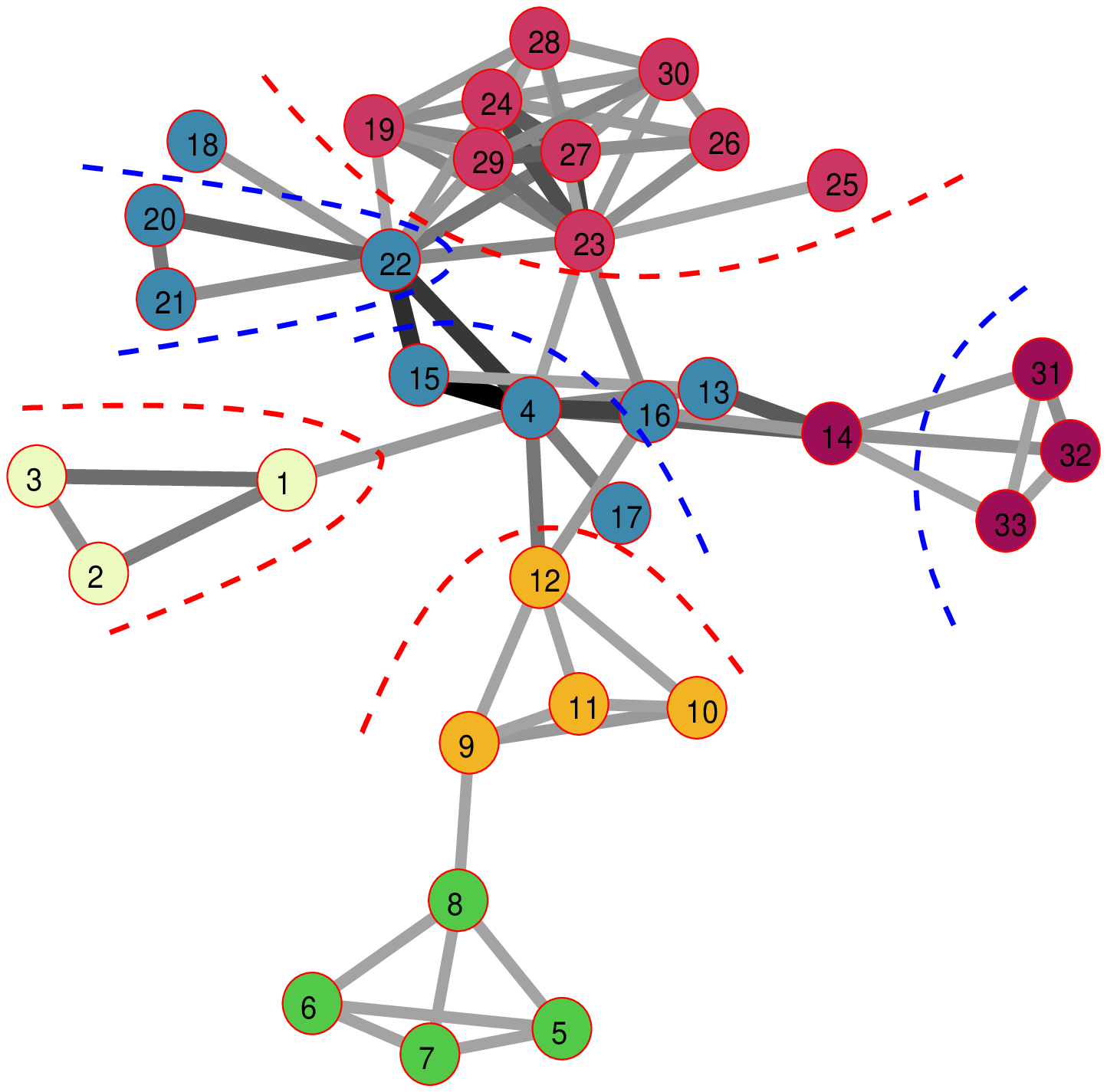} \\
(a) & (b) \\
\includegraphics[width=80mm]{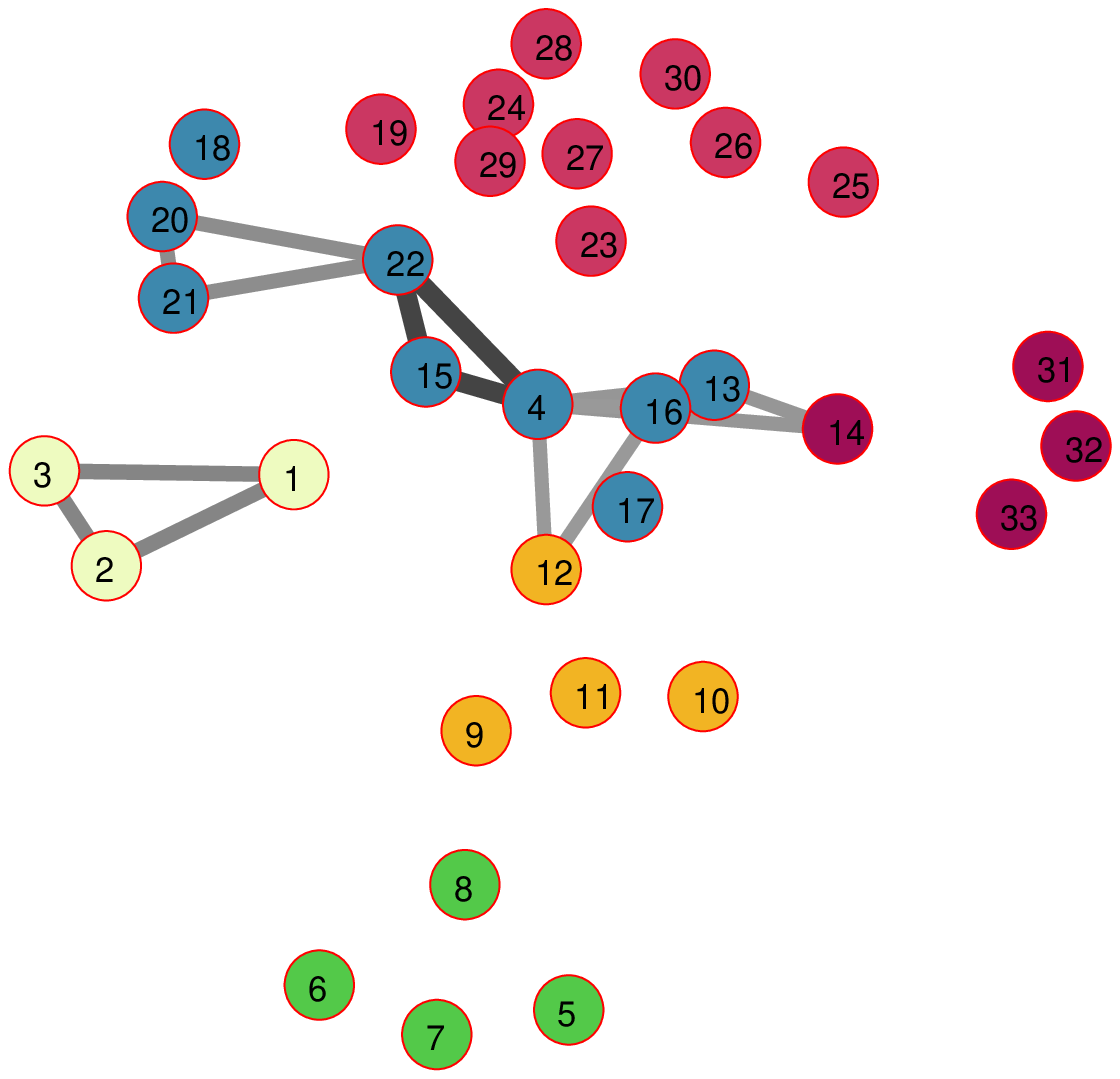} & 
\includegraphics[width=80mm]{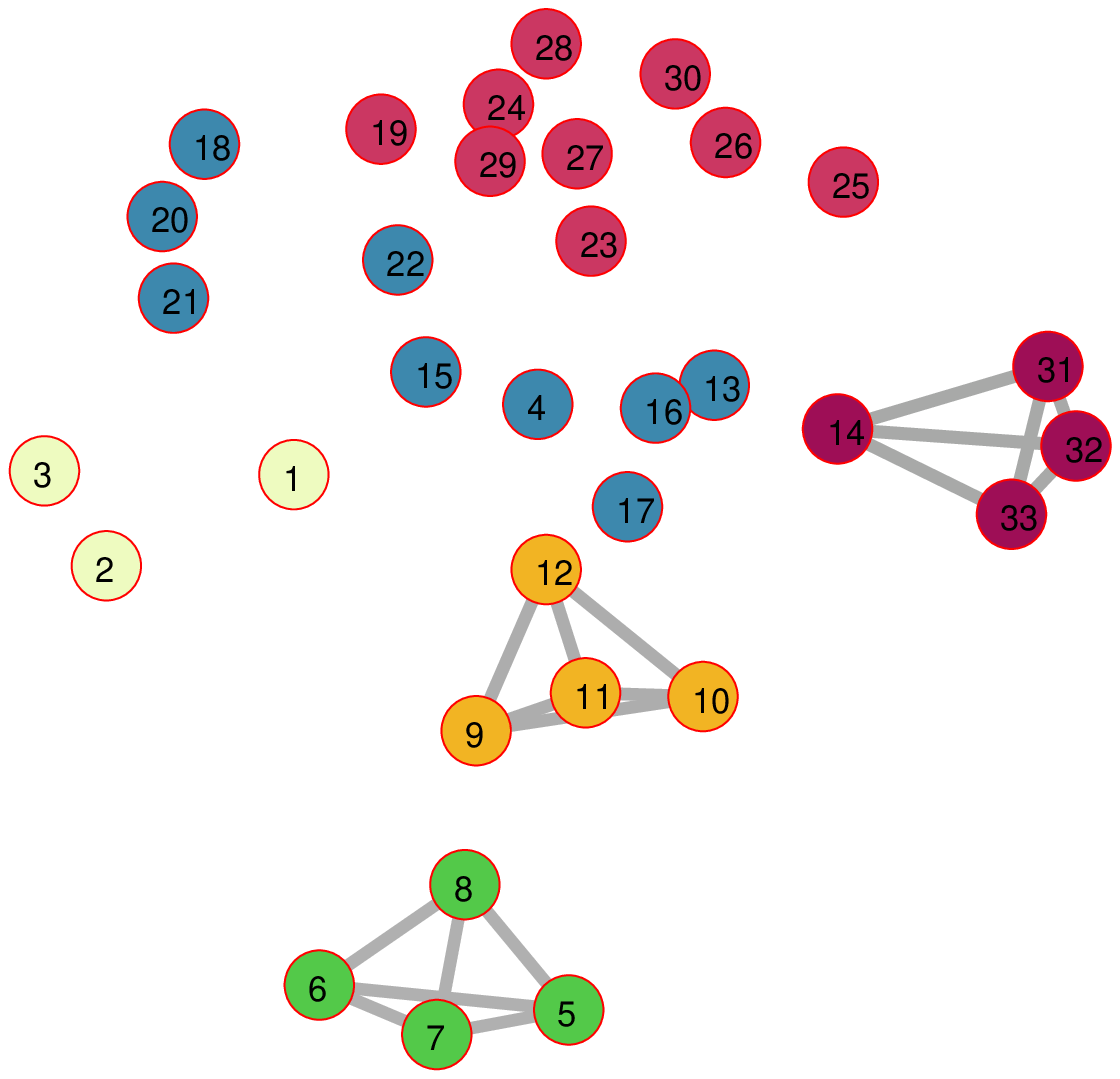} \\
(c) & (d) 
\end{tabular}
\end{center}
\caption{\label{fig:Lesmis} Decomposition of {\sl Les Miserables} social network. 
(a) Social network of characters in {\sl Les Miserables}; (b) Spectral clustering result; (c) The identified 3-cliques; 
(d) The identified 4-cliques.}
\end{figure*}
The results are shown in Figure \ref{fig:Basketball}.  In Figure \ref{fig:Basketball}-(a), we see that the two basketball teams are perfected detected as expected. Since the two $5$-sets correspond to the two teams have no overlap, hence satisfy the irrepresentable Condition (IRR). 
In Figure \ref{fig:Basketball}-(b), we try to detect the two teams under different noise levels $\epsilon \in [0, 1]$. 
The two basketball teams can be detected under fairly large noise levels.
This example can also be dealt with using spectral clustering techniques where we normalize the pairwise interaction data to get
the transition matrix, followed by spectral clustering on eigenspaces. We observed that
both our method and spectral clustering works very well under noise level less than $0.8$ (i.e. $|\epsilon|<0.8$).

\subsection{The Social Network of Les Mis\`{e}rables}

We consider the social network of  $33$ characters in Victor Hugo's novel {\sl Les Mis\`{e}rables}~\citep{knuth1993}. We represent this social network using a weighted graph (Figure \ref{fig:Lesmis}-(a)). The edge weights
are the co-appearance frequencies of the two corresponding characters.
Table 1 illustrates several social communities formed by relationships including {\it friendships, street gangs, kinships}, etc. The underlying social community, regarded as the ground truth for the data, is summarized in Figure 
\ref{fig:Lesmis}-(a) where several social communities arise. 
Figure \ref{fig:Lesmis}-(b) shows the spectral clustering result  in  which the first 
three red cuts are reasonable while 
the next three blue cuts destroyed a lot of community structures within the network.
\begin{figure*}[ht!]
\begin{center}
\begin{tabular}{cc}
\includegraphics[width=80mm]{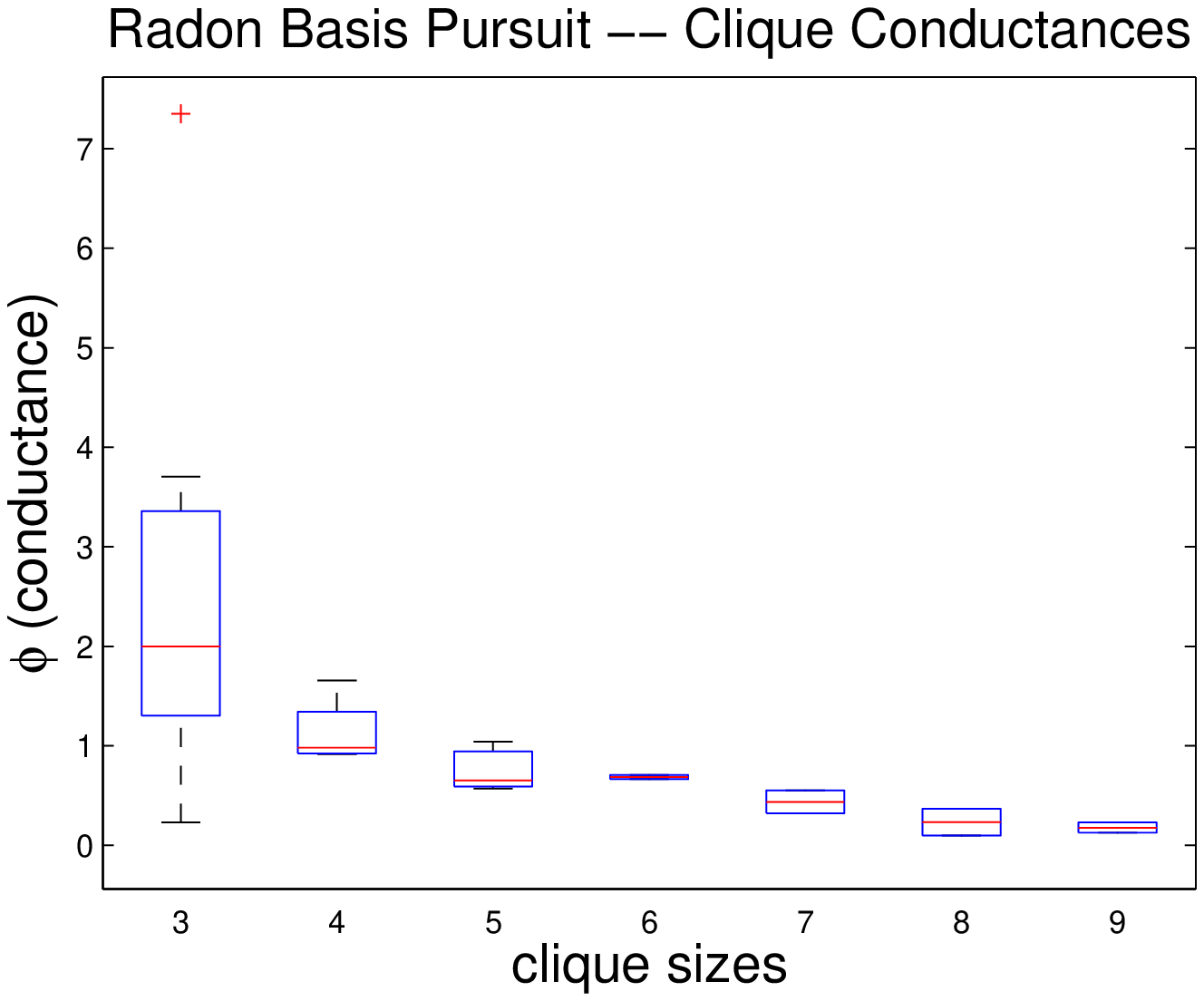} &
\includegraphics[width=80mm]{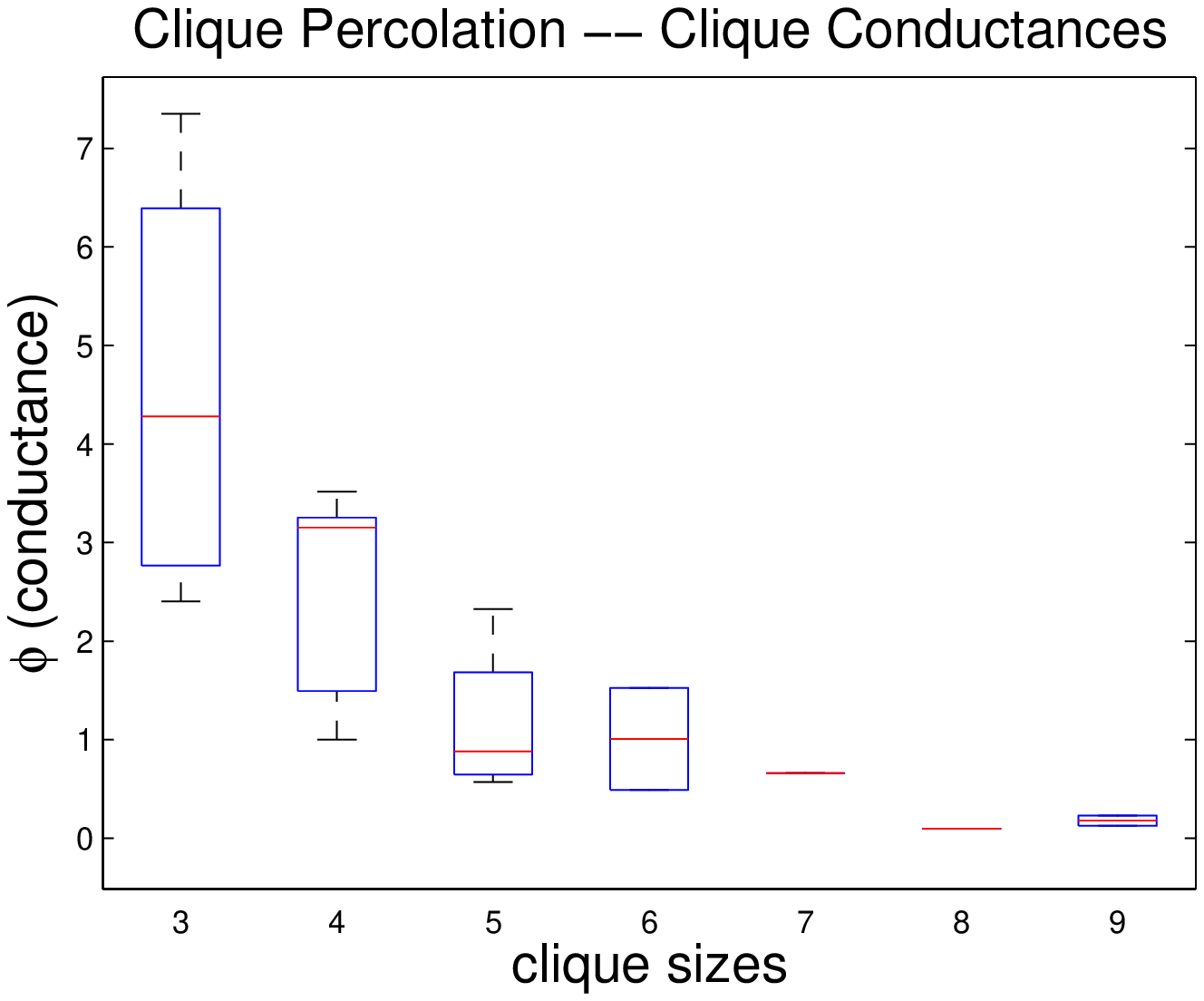} \\
(a) & (b) \\
\includegraphics[width=80mm]{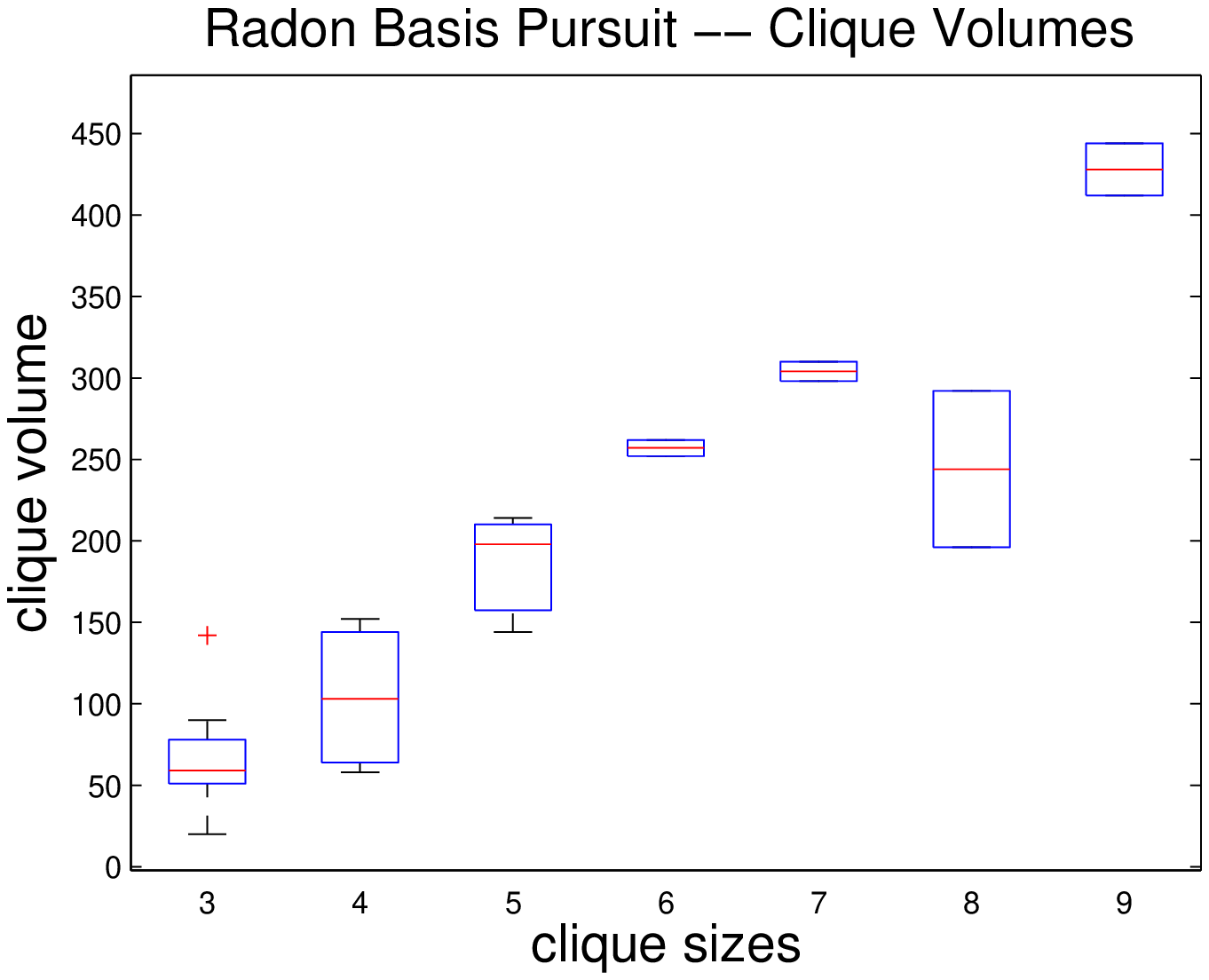} &
\includegraphics[width=80mm]{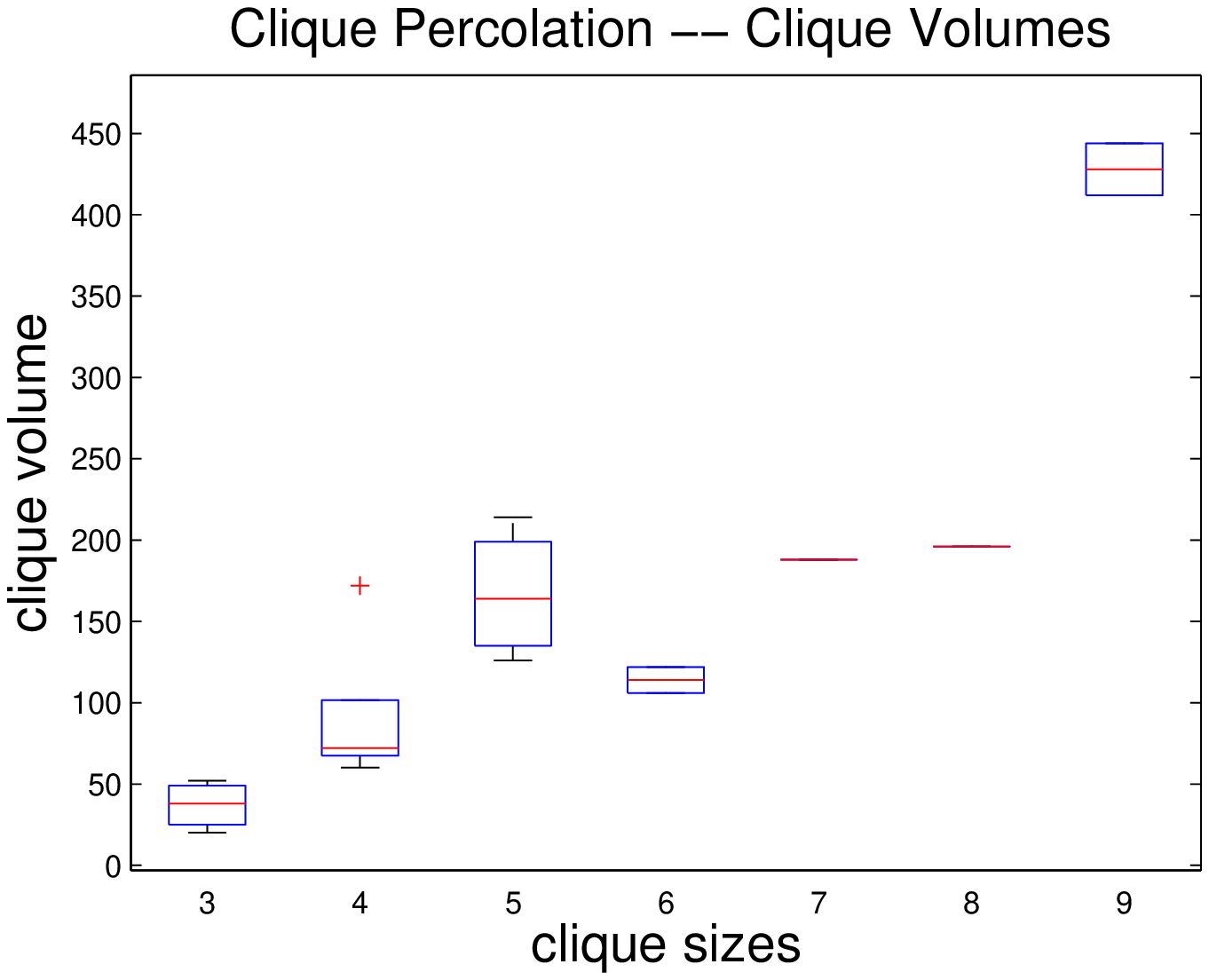} \\
(c) & (d) \\
\end{tabular}
\end{center}
\caption{\label{fig:LesmisStat} {\sl Les Miserables} social network:  Box plot of clique conductances and volumes for clique percolation method and our approach. Cliques identified by our approach have smaller conductances and larger volumes. }
\end{figure*}

We compare our method with the clique percolation method, $23$ and $19$ cliques were identified
respectively where our approach can identify more meaningful cliques --
see Figure \ref{fig:Lesmis} and Table \ref{table:lesmistable} where we verified the ground truth
from the novel.
For example, our method can correctly identify two separate cliques $\{4,15,22\}$ and $\{20,21,22\}$,
while the clique percolation method treats $\{4,15,20,21,22\}$ as a single clique. The interaction frequencies among those
characters, however, show that there are relatively smaller cross-community interactions, thus those two $3$-cliques should be separated.
Figure \ref{fig:Lesmis}-(c) and \ref{fig:Lesmis}-(d) depict important $3$-cliques and $4$-cliques identified by our algorithm. 
The sparsity patterns of those cliques satisfy the irrepresentable condition where overlaps between them are generally not large. 
However, they do not necessarily satisfy the condition in Lemma \ref{lemma:worstcase} which is 
based on a worst-case analysis. 
In Figure \ref{fig:LesmisStat}, we also compare both methods in terms of clique conductances and volumes and see that the cliques identified by Radon basis pursuit have slightly lower conductances and
larger volumes, which demonstrates advantages of our approach.

\begin{table}[ht]\footnotesize
\centering
\caption{\label{table:lesmistable}Social Networks of Les Mis\`{e}rables\vspace{0.1in}}
\begin{tabular}{c|c|c|c|c}
\hline\hline
Cliques & Names of Characters & Relationships & Perco. & Radon \\ [0.5ex]
\hline
$\{1,2,3\}$ & \{Myriel, Mlle Baptistine, Mme Magloire\} & Friendship & N & N\\
$\{4,13,14\}$ & \{Valjean, Mme Thenardier, Thenardier\} & Dramatic Conflicts & N & Y \\
$\{4,15,22\}$ & \{Valjean, Cosette, Marius\} & Dramatic Conflicts & N & Y\\
$\{20,21,22\}$ & \{Gillenormand, Mlle Gillenormand, Marius\} & Kinship & N & Y\\
$\{5,6,7,8\}$ & \{Tholomyes, Listolier, Fameuil, Blacheville\} & Friendship & Y & Y\\
$\{9,10,11,12\}$ & \{Favourite, Dahlia, Zephine, Fantine\}  & Friendship & Y & Y\\
$\{14,31,32,33\}$ & \{Thenardier, Gueulemer, Babet, Claquesous\} & Street Gang & N & Y\\
\hline
\end{tabular}
\end{table}

In summary,  our method obtains more abundant social structure information than the competing 
techniques. We  also obtain social communities with overlaps which is impossible for clustering methods. 
We note that some simple schemes will not work well. For example, one may think of scoring each large 
clique by the mean scores of the included small cliques. 
In this example, since two or three key characters appear very frequently, 
we will end up with finding that the top high order cliques always contain them. In fact, among the top ten 3-cliques,
seven of them contain node $4$ and six of them contain node $15$, which does not give us good results.

\subsection{Coauthorships in Network Science}

We also studied a medium size coauthorship network where there is a total of
1,589 scientists who come from a broad variety of fields. Part of this network is shown in Figure \ref{fig:Netsci}-(a).
$136$ and $166$ cliques are identified by our approach and the clique percolation method respectively. We also compare the two methods
in terms of clique conductances and volumes.
From Figure \ref{fig:NetsciStat}-(a),(b), we see that the cliques identified
by Radon basis pursuit have smaller conductances and comparable clique volumes than the clique percolation method.
Our approach can scale very well. In this example, it can identify the cliques up to size $9$ in $564$ seconds.
So this application example shows that our approach can be used to identify cliques in social networks with 
hundreds or even thousands of nodes.

\begin{figure}[ht]
\begin{center}
\begin{tabular}{cc}
\hspace{-1cm}\includegraphics[width=100mm]{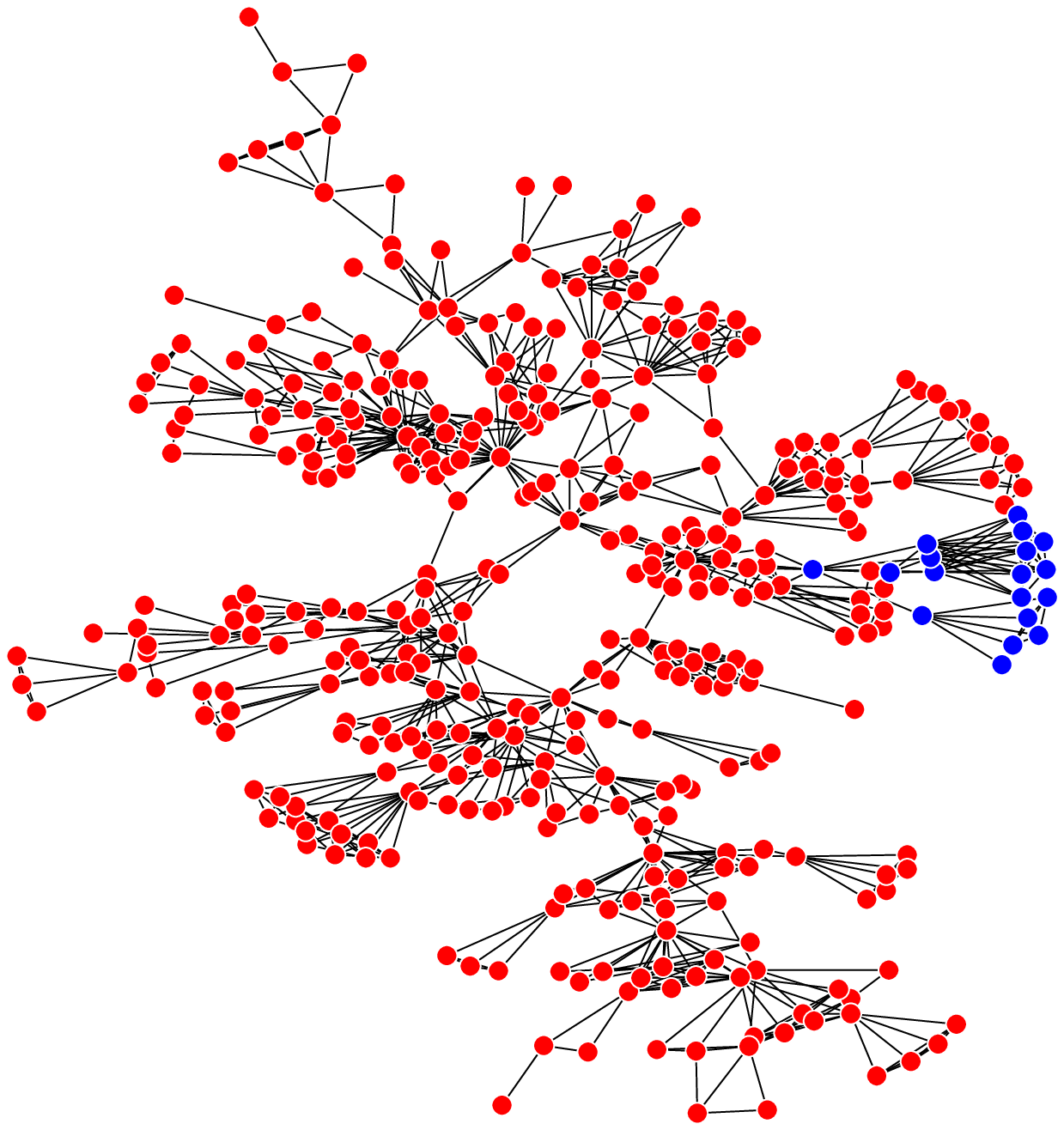} &
\hspace{-1cm}\includegraphics[height=85mm]{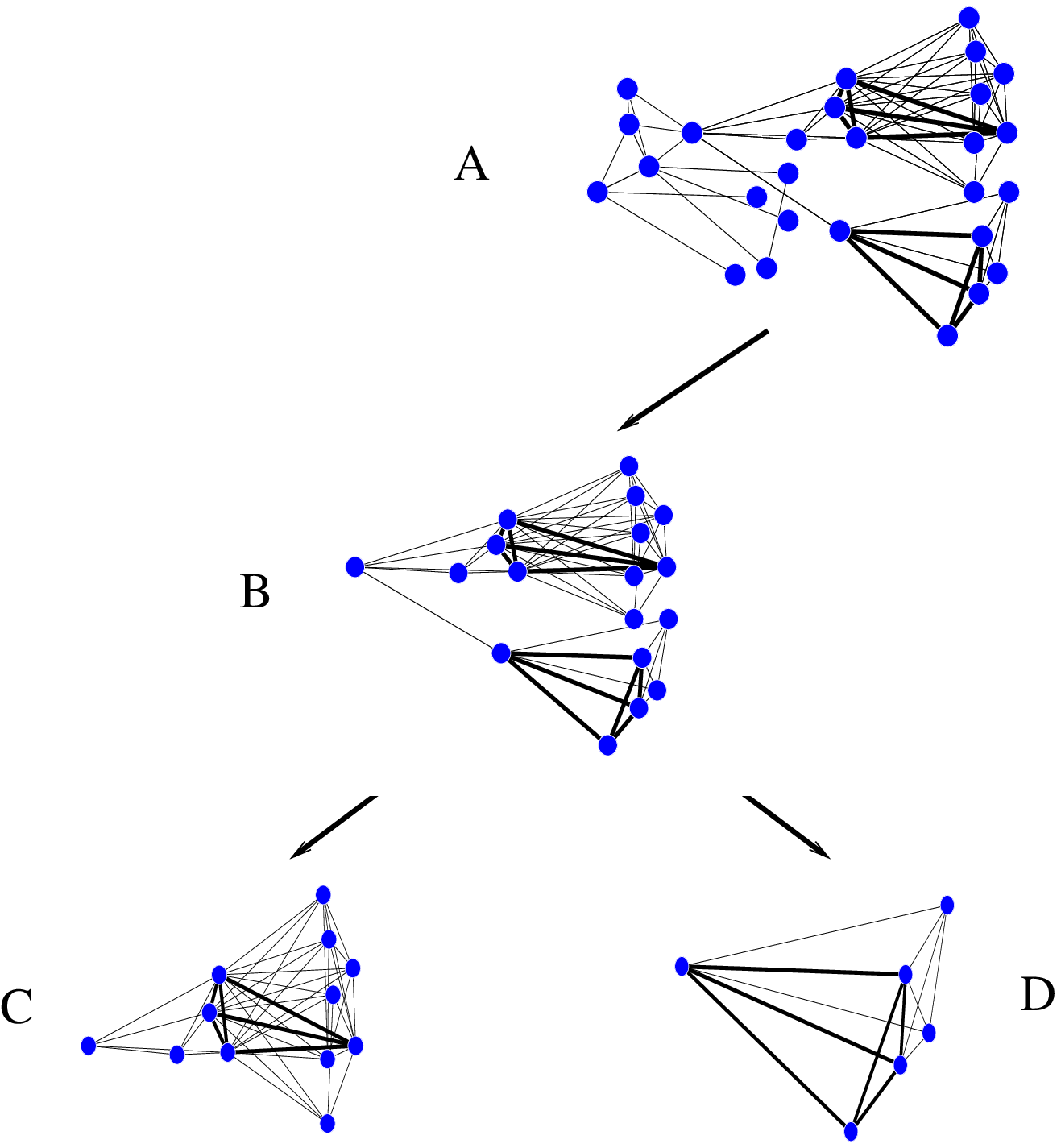}\\ 
(a) & (b)  
\end{tabular}
\end{center}
\caption{\label{fig:Netsci} (a) Coauthorships in Network Science, only a part of the network is shown; (b) Important cliques identified within clusters behave in a 
persistent way. Clustering node B is exactly the blue part in (a) } 
\end{figure}

Finally, we note that clustering techniques, e.g., spectral clustering, combined with our algorithm can provide a more  refined analysis of the network. 
We can look at the persistence of identified cliques in the binary tree
decomposition of bipartite spectral clustering of the network in a bottom-up way. 
Cliques which persist through more levels will give us meaningful community structural information.

In figure \ref{fig:Netsci}-(b), a small fraction of the binary tree decomposition of bipartite 
spectral clustering is depicted, 
where child nodes are spectral bipartition of the parent node. We can detect cliques 
within the child nodes. Once cliques within clusters $C, D$ are identified, we 
then backtrack to the parent nodes $B$ and $A$ to see if the identified cliques still persist.

\begin{figure*}[ht!]
\begin{center}
\begin{tabular}{cc}
\includegraphics[width=80mm]{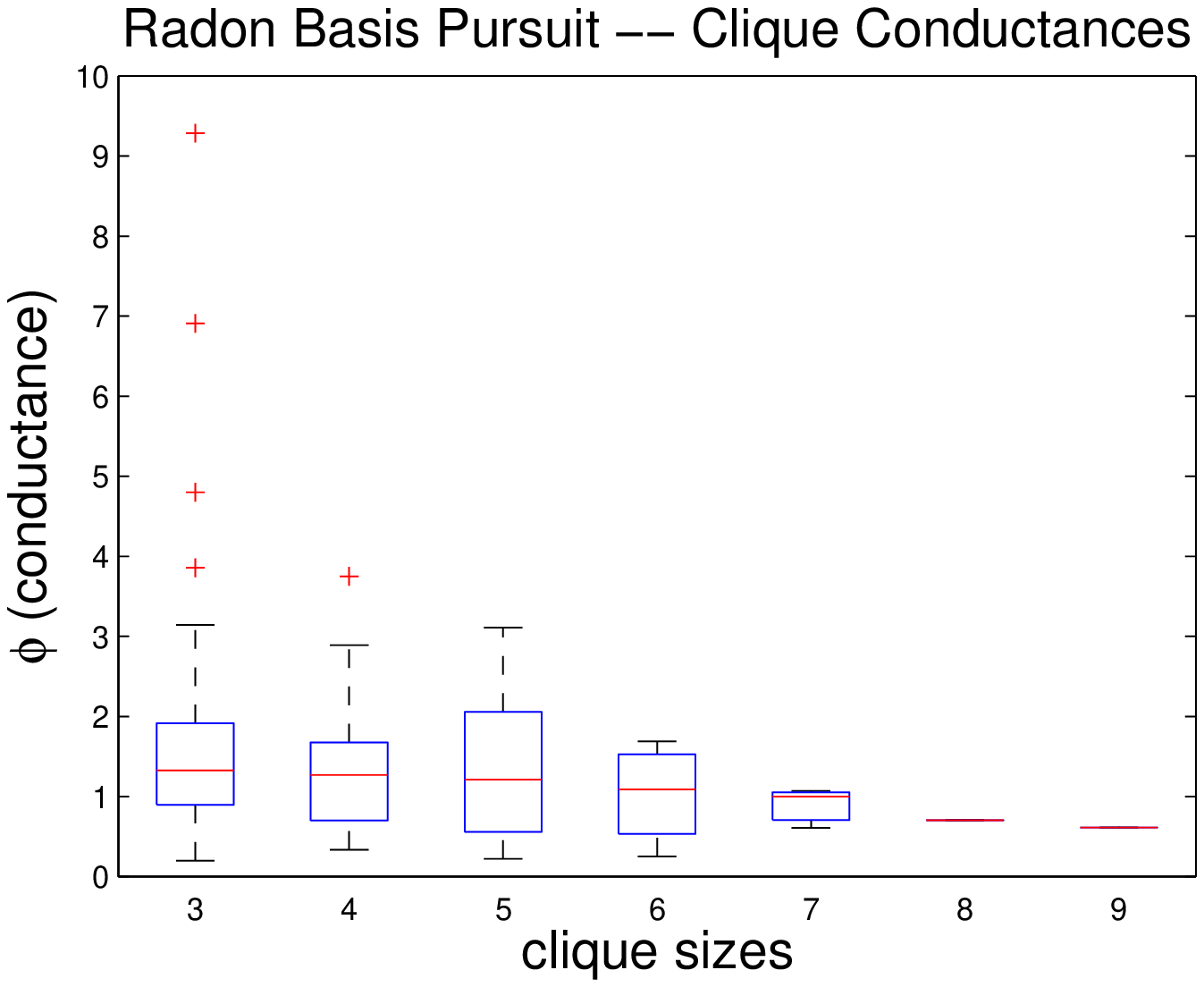} &
\includegraphics[width=80mm]{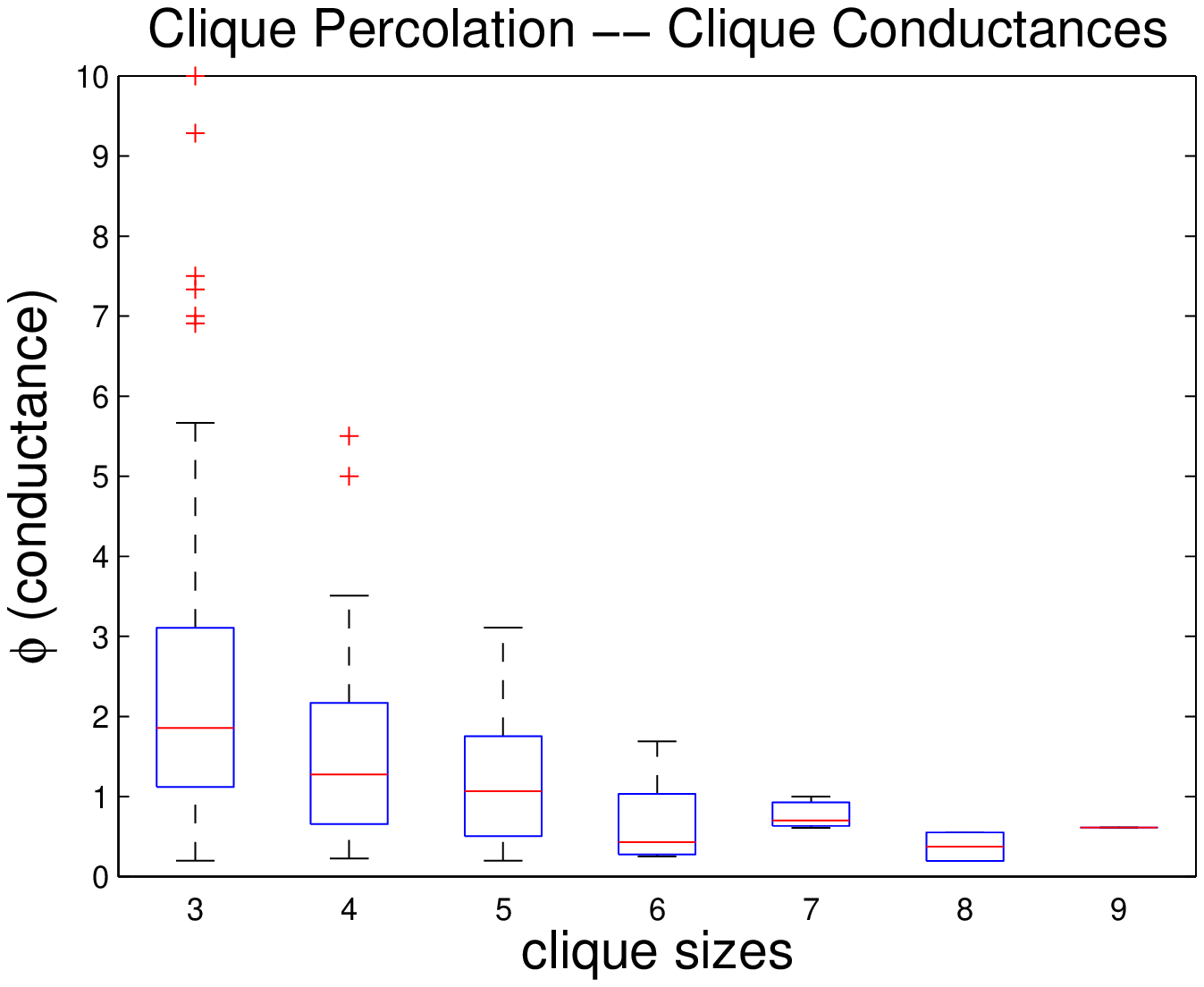} \\
(a) & (b) \\
\includegraphics[width=80mm]{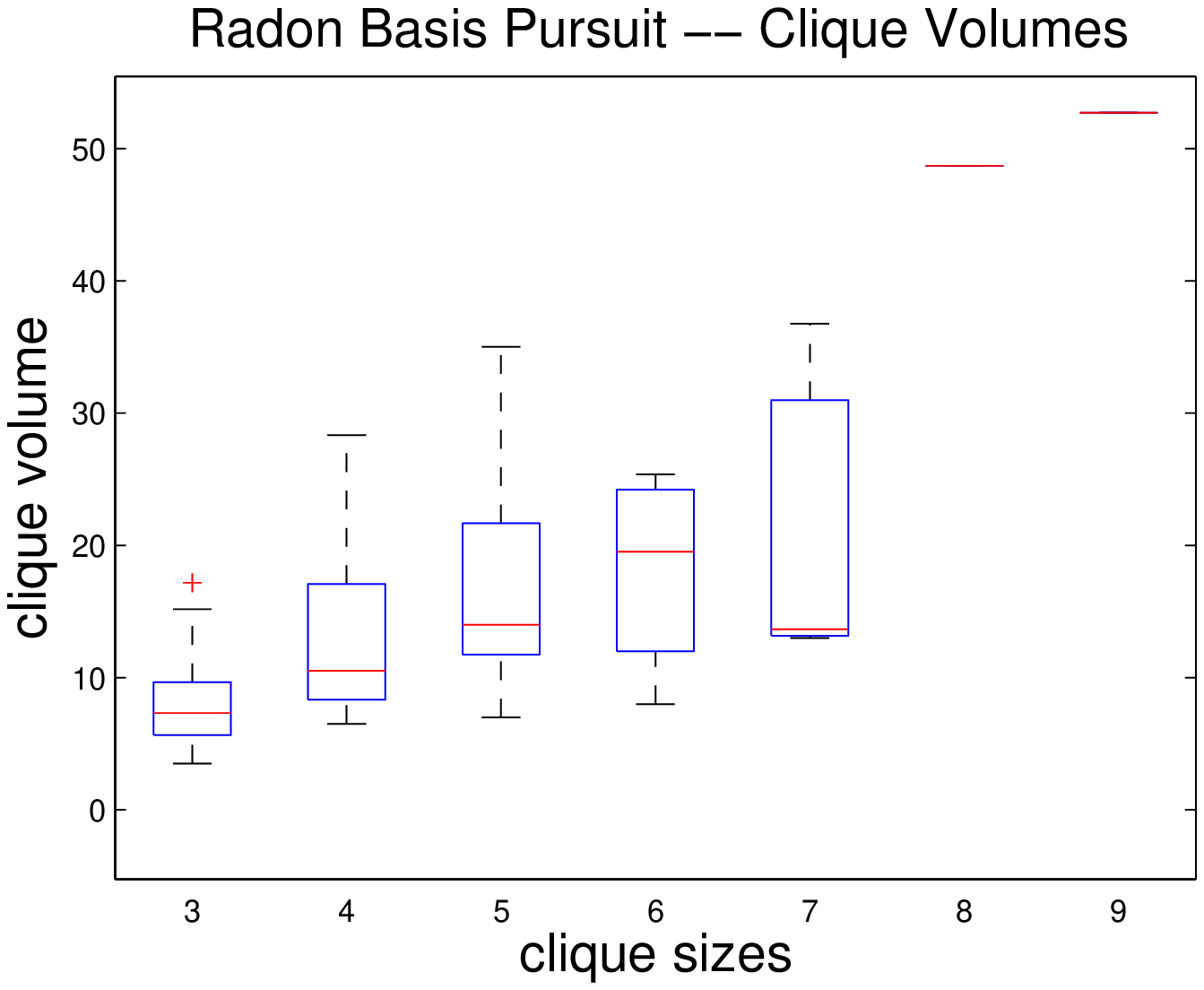} &
\includegraphics[width=80mm]{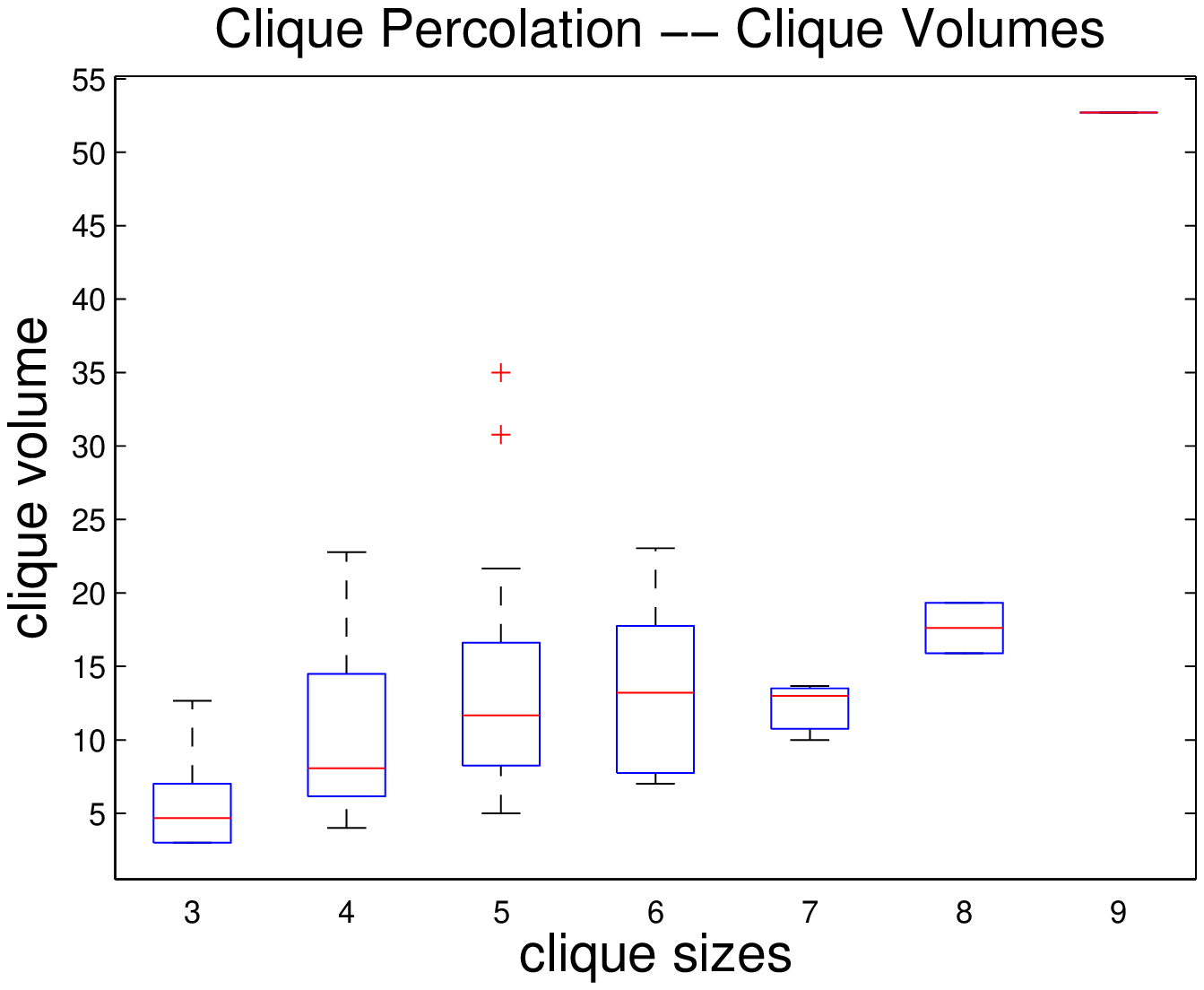} \\
(c) & (d) \\
\end{tabular}
\end{center}
\caption{\label{fig:NetsciStat} Coauthorship Network: Box plot of clique conductances and volumes for clique percolation method and our approach. Cliques identified by our approach have smaller conductances
and larger volumes. }
\end{figure*}

We can identify $3$ cliques ($c_1$=\{Kumar, Raghavan, Rajagopalan, Tomkins\}, $c_2$=\{Kumar. S, Raghavan, Rajagopalan\}, 
$c_3$=\{Raghavan, Rajagopalan, Tomkins, Kumar. S\}) 
within $C$ and $3$ cliques ($d_1$=\{Flake. G, Lawrence. S, Giles. C, Coetzee. F\}, $d_2$=\{Flake. G, Lawrence. S, 
Giles. C, Pennock. D, Glover. E\}, $d_3$=\{Flake. G, Lawrence. S, Giles. C\}) within $D$ which persist to parents $B$ and $A$. We can identify papers whose authors are exactly those cliques. 
Using only clustering will not get this result because those cliques have heavy overlaps between them.

In figure \ref{fig:Netsci}-(b), for simplicity, 
we only show two persistent cliques:  
$c_1$=\{Kumar, Raghavan, Rajagopalan, Tomkins\} and
$d_1$=\{Flake. G, Lawrence. S, Giles. C, Coetzee. F\} which are the most important cliques (having the largest weights when solving the
LP program) in clusters $C$ and $D$ respectively. These two cliques are also the most important two cliques in 
cluster $B$, and if we even further back track them to clustering $A$, they are still ranked as the first and the third
in terms of weights among all cliques identifiable in $A$.

\subsection{Inferring high order ranking}
Jester dataset \citep{jester2001} contains about $24,000$ users who give ratings
on $100$ jokes. Those ratings are of real value ranging from $-10.00$ to 
$+10.00$. We extract top $20$ jokes from the entire dataset according to mean scores.
Among those $20$ jokes, we count the voting on top $5$-jokes by each user and view them as the ground truth. Figure \ref{fig:jester}-(a) shows that there is a top $5$-set, $\{27,29,35,36,50\}$, with an overwhelming voting than the others.

Now suppose we only know information as top $3$ counts of the jokes and wonder if we can identify the most popular 5-joke group. By solving $\mathcal{P}_{1,\delta}$ with the whole 
regularization path by varying $\delta$, we are capable to detect this subset (Figure \ref{fig:jester}-(b)) in a robust way.

\begin{figure}[htp]
\begin{center}
\begin{tabular}{cc}
\includegraphics[width=0.5\textwidth]{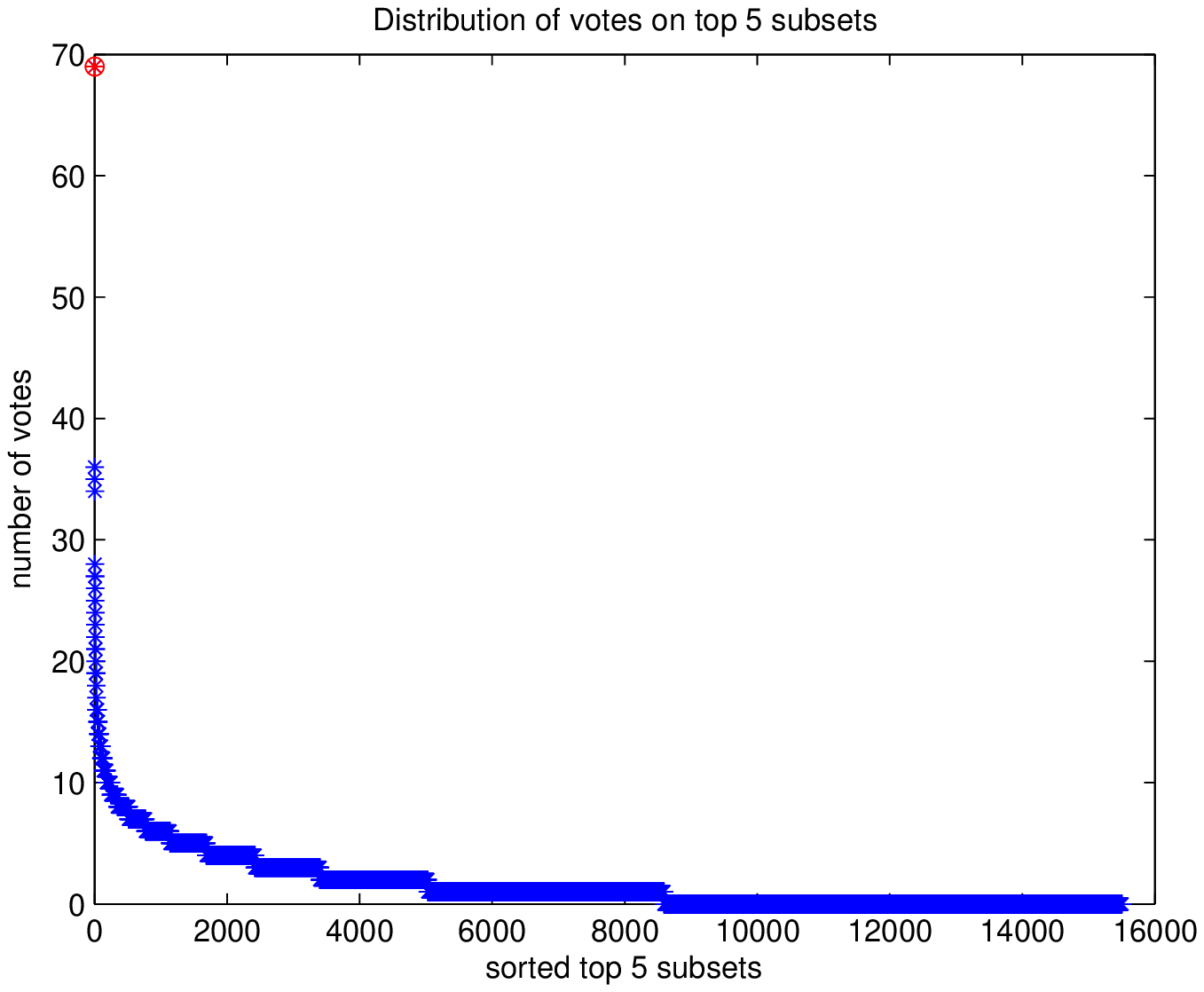} & 
\hspace{-0.4in}\includegraphics[width=0.5\textwidth]{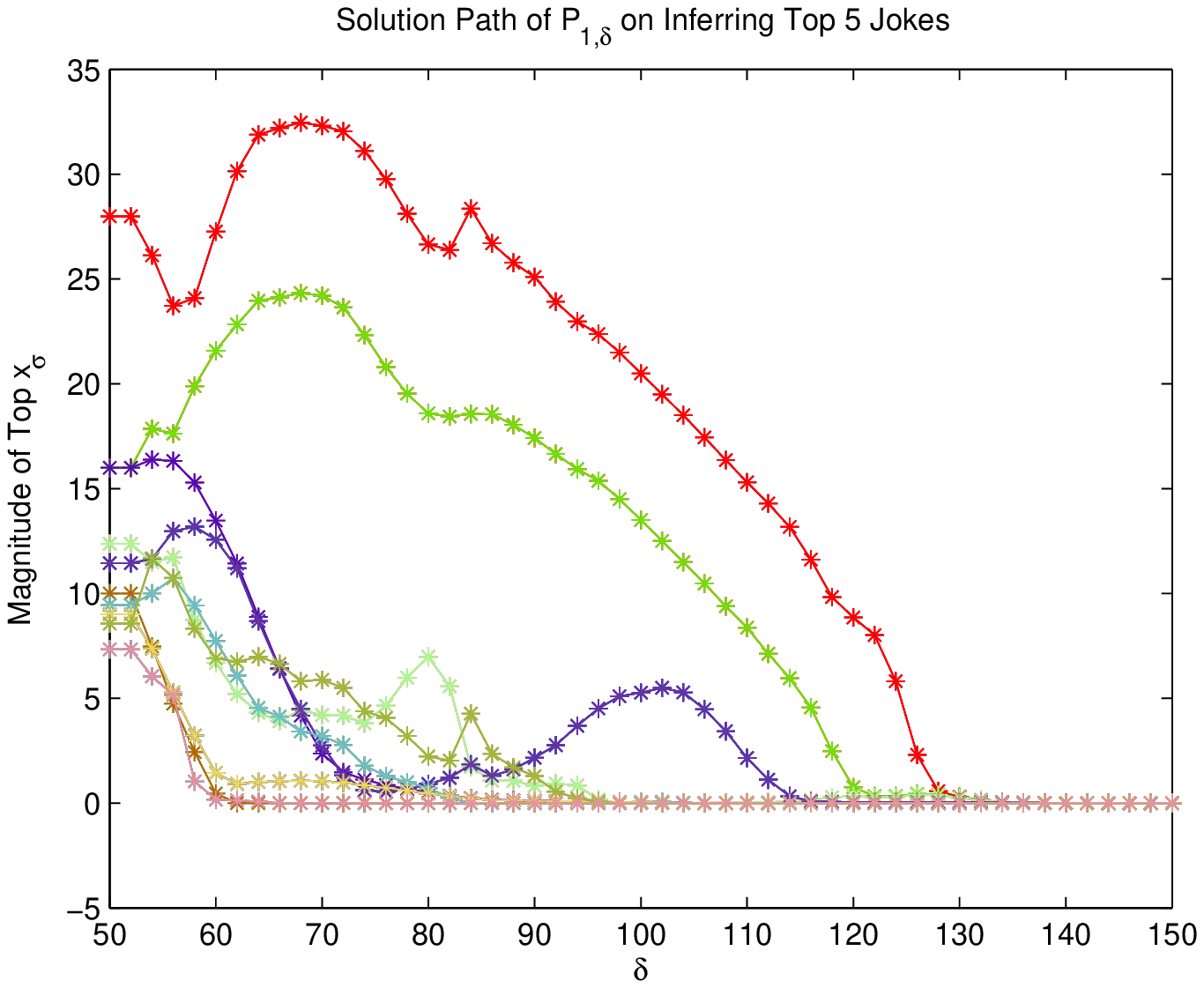} \\
(a) & (b) \\ 
\end{tabular}
\caption{ \label{fig:jester} (a) ÊThere is a significant top-5 jokes (in red) whose ID is $\{27,29,35,36,50\}$; (b) Regularization path where the top curve (red) selects this top group over $\delta\in [50,130]$. Note that the top $2^{nd}$ curve (green) also identifies the fourth 5-set in a persistent way. }
\end{center}
\end{figure}

\section{Conclusions}\label{sec:conclusion}

In this work, we present a novel approach to connect two seemingly different areas: {\it network data analysis} and {\it compressive sensing}.  By adopting a new algebraic tool, {\it Randon basis pursuit in homogeneous spaces}, we formulate the network clique detection problem into a compressed sensing problem.  Such a novel formulation allows us to construct rigorous conditions to  characterize the network clique recovery problems.   Instead of providing another heuristic method, we aim at contributing at the foundational level to network data analysis.  We hope that our work could build a bridge connecting the research communities of network modeling and compressive sensing, so that research results and tools from one area could be ported to another one to create more exciting results.

To illustrate the usefulness of this new framework, we present a novel approach to identify overlapped communities as cliques in social networks, based on compressed sensing with an new algebraic method, i.e. Radon basis pursuit in homogeneous spaces associated with permutation groups. Our approach starts from a general problem of compressive representation of low order interactive information from high order cliques, which firstly arises from identity management and statistical ranking, etc. Specifically applied to social networks, this approach studies bi-variate functions defined on pairs of nodes, and looks for compressive representations of such functions based on clique information in networks. It turns out that the sparse representation under Radon basis may disclose community structures, typically overlapped, in social networks. We have shown that noiseless exact recovery and stable recovery with uniformly bounded noise hold under some natural conditions. Though this paper is mainly methodological and theoretical, we also develop a polynomial-time approximation algorithm for solving empirical problems and demonstrate the usefulness of the proposed approach on real-world  networks.

\section{Acknowledgments}

Xiaoye Jiang and Leonidas Guibas wish to acknowledge the support of ARO grants W911NF-10-1-0037 and W911NF-07-2-0027, as well as NSF grant CCF 1011228 and a gift from the Google Corporation. Y. Yao acknowledges supports from the National Basic Research Program of China (973 Program 2011CB809105), NSFC (61071157), Microsoft Research Asia, and a professorship in the Hundred Talents Program at Peking University. The authors also thank Zongming Ma, Minyu Peng, Michael Saunders, Yinyu Ye for very helpful discussions and comments. Han Liu is  thankful for a faculty supporting package from Johns Hopkins University.


\bibliography{sample}

\end{document}